\documentclass{article} 
\usepackage{iclr2027_conference}

\usepackage{amsmath,amsfonts,bm}

\def\eqref#1{equation~\ref{#1}}

\def\1{\bm{1}}

\DeclareMathAlphabet{\mathsfit}{\encodingdefault}{\sfdefault}{m}{sl}
\SetMathAlphabet{\mathsfit}{bold}{\encodingdefault}{\sfdefault}{bx}{n}

\DeclareMathOperator{\sign}{sign}
\DeclareMathOperator{\Tr}{Tr}

\usepackage{hyperref}
\usepackage{times}
\usepackage{url}
\usepackage{graphicx}
\usepackage{algorithm}
\usepackage{algorithmic}
\usepackage{booktabs}
\usepackage{comment}
\usepackage{wrapfig}
\usepackage{multirow}
\usepackage{amsthm}

\theoremstyle{plain}
\newtheorem{theorem}{Theorem}
\newtheorem{proposition}{Proposition}
\newtheorem{lemma}{Lemma}
\newtheorem{corollary}{Corollary}
\theoremstyle{definition}
\newtheorem{definition}{Definition}

\theoremstyle{remark}
\newtheorem{remark}{Remark}

\title{Sparse-TD-JEPA: Improving Zero-Shot Reinforcement Learning With Sparse Latent Task Representations}

\author{Duaa Nakshbandi\\
University of California, Merced\\
\texttt{duaanakshbandi@ucmerced.edu} \\
\And
Zhiyu An\\
University of California, Merced\\
\texttt{zan7@ucmerced.edu} \\
\And
Wan Du\\
University of California, Merced\\
\texttt{wdu3@ucmerced.edu} \\
}

\iclrfinalcopy 
\begin{document}

\maketitle

\begin{abstract}
Successor-feature-based zero-shot reinforcement learning methods pre-train a task encoder and successor-feature predictor on offline, reward-free data, then recover a policy at test time by solving a closed-form linear regression over a small labeled dataset. 
We identify an underexplored limitation of this approach: when the task encoder produces dense representations, the inferred reward signal becomes diffuse across the state space, reducing the contrast between goal-relevant states and the background and providing a weaker signal for the policy to follow. 
We address this limitation by introducing an $\ell_1$ sparsity penalty that concentrates reward-predictive structure into a smaller effective subspace of the task encoder's latent space. 
The penalty is applied only during training and requires no architectural change or modification to the closed-form inference procedure. 
We instantiate this approach as \emph{Sparse-TD-JEPA}, a single-line of code addition based on Temporal-Difference Joint-Embedding Predictive Architecture (TD-JEPA) that significantly improved its performance in zero-shot reinforcement learning. 
We evaluate Sparse-TD-JEPA on 65 tasks across the DeepMind Control (DMC) Suite and OGBench, covering locomotion, navigation, trajectory stitching, puzzle solving, and robotic manipulation.
Sparse-TD-JEPA improves over TD-JEPA in all 13 domain-level comparisons. 
On DMC, it increases the average return from $661 \pm 6$ to $714 \pm 3$. 
On OGBench, it increases the average success rate from $38.0 \pm 0.8$ to $48.4 \pm 2.3$, with improvements of 13.4--18.0 percentage points on the evaluated antmaze navigation and stitching domains. 
Overall, Sparse-TD-JEPA achieves the highest mean performance among established zero-shot baselines in nine of the 13 domains. 
Analysis further shows that sparsity produces a sharper and more task-informative inferred reward landscape, offering a mechanistic explanation for the observed performance gains.
\end{abstract}

\section{Introduction}

Zero-shot reinforcement learning studies the problem of pre-training a single policy on offline, reward-free data such that it can solve any downstream task without further interaction with the environment. 
It is important because it scales the achievable tasks of a single model without requiring an equal amount of scaling to the training data or online experience, making RL more practical across the diverse and open-ended settings encountered in the real world.

A prominent line of work addresses this problem through successor features and task encoders, where a latent representation of states is learned such that any reward function can be expressed as a linear function of that representation \cite{barreto2017successor, touati2021learning}. At test time, a small set of labeled transitions suffices to recover a task vector via linear regression, which is then used to retrieve the corresponding pre-trained policy. 
Recent zero-shot RL methods have explored a diverse range of objectives for constructing task representations. 
Laplacian representations capture the geometry of the environment through eigenfunctions of its transition graph \cite{wu2018laplacian}, while ICVF learns representations through intention-conditioned value prediction \cite{ghosh2023reinforcement}. 
HILP instead organizes the latent space around temporal distances for goal reaching \cite{park2024foundation}. 
Self-predictive approaches such as BYOL and BYOL-$\gamma$ have also been adapted to the successor-feature setting by predicting future latent representations over one or multiple time scales \cite{grill2020bootstrap,lawson2025self}. 
More recent methods, including RLDP and TD-JEPA, improve multi-step latent-dynamics prediction through chained or temporal-difference objectives \cite{jajoo2026regularized,bagatella2025td}. 
Although these approaches differ in how they learn task-relevant structure, they all ultimately rely on the learned task representation to support accurate linear reward inference at test time. 

Despite this progress, a critical bottleneck has received comparatively little attention. When the task encoder spreads reward-predictive information diffusely across all latent dimensions, the predicted reward landscape is less concentrated near the goal. 
The resulting reward signal has a weaker gradient to follow, leading to degraded downstream performance. We find that this difference in landscape sharpness persists across a wide range of labeled-set sizes, suggesting it reflects a structural property of the encoder geometry.

We address these challenges through a sparsity penalty for the task encoding representation to concentrate its reward-predictive structure into a smaller effective subspace of its latent space. 
We instantiate this idea as \emph{Sparse-TD-JEPA}, a single-line code addition based on TD-JEPA framework that significantly improves performance on various navigation, locomotion, and puzzle tasks.
Empirically, this simple modification produces consistent gains across all 13 evaluated domain-level comparisons. 
On \textit{DeepMind-Control (DMC)}, Sparse TD-JEPA improves the average return from $661 \pm 6$ to $714 \pm 3$, a gain of 53 points. 
On \textit{OGBench}, it raises the average success rate from $38.0 \pm 0.8$ to $48.4 \pm 2.3$, an absolute improvement of 10.4 percentage points. 
Moreover, analysis experiment shows that the resulting encoder produces a sharper, more peaked predicted reward landscape, giving the downstream policy a clearer signal to follow. 

Our contributions are as follows:
\begin{itemize}
    \item We demonstrate theoretically and empirically that $\ell_1$ sparsity on the task encoder improves zero-shot reinforcement learning performance.
    \item Through controlled analysis experiments, we show that the sparsity-penalized task encoder produces a smoother and more informative reward landscape.
    \item We evaluate our method across standard zero-shot reinforcement learning benchmarks covering robotic locomotion, navigation, and manipulation tasks. We  demonstrate consistent gains over the base method and other established zero-shot baselines.    
\end{itemize}

\section{Related Work}
\label{sec:related-work}

\paragraph{Latent-predictive representation learning.}
The Joint-Embedding Predictive Architecture (JEPA) \cite{lecun2022path} is a representations learning architecture that splits an input into a context view and a target view, then trains a predictor to map the context encoding to the target encoding in a shared latent space.  This principle has been instantiated across modalities. Image-based Joint-Embedding Predictive Architecture (I-JEPA) \cite{assran2023self} applies it to static images. It uses an asymmetric context/target-encoder design to predict the representations of masked image regions from visible patches. Video-based Joint-Embedding Predictive Architecture (V-JEPA) \cite{bardes2024revisiting, assran2025v} extends this to video. It predicts masked spatio-temporal tokens from unmasked ones and supports downstream robotic planning without task-specific reward. Temporal-Difference Joint-Embedding Predictive Architecture (TD-JEPA) \cite{bagatella2025td} instantiates the same context-target prediction problem in reinforcement learning. A state encoder and a task encoder are trained via a temporal-difference latent-predictive loss over offline, off-policy transitions. The predictor is conditioned on the action and a policy latent.

\paragraph{Zero-shot reinforcement learning.}
Zero-shot Reinforcement Learning (RL) methods are methods that pre-train an agent once on offline, reward-free data, then at test time solve arbitrary downstream tasks without further interaction. Most zero-shot RL methods learn a \emph{task encoder} alongside \emph{successor features}. The task encoder maps states to a space in which any reward is linear. The \emph{successor features} predict how much of each task-encoder dimension a policy accumulates over time. Forward-Backward (FB) representations \cite{touati2021learning}, a zero-shot RL method, learns both jointly via a bilinear decomposition of successor measures. Other zero-shot reinforcement learning methods construct the task encoder through different objectives. Bootstrap Your Own Latent (BYOL) \cite{grill2020bootstrap} predicts an encoder's own future latents. The Laplacian representation \cite{wu2018laplacian} uses eigenfunctions of the environment's transition graph. The Intention-Conditioned Value Function (ICVF) \cite{ghosh2023reinforcement} method uses an implicit value-based objective. HILbert foundation Policy (HILP) \cite{park2024foundation} enforces distance-preservation for goal-reaching. Regularized Latent Dynamics Prediction  (RLDP) \cite{jajoo2026regularized} uses chained multi-step latent prediction. BYOL-$\gamma$ \cite{lawson2025self} extends BYOL to multi-step targets. TD-JEPA \cite{bagatella2025td} uses a temporal-difference latent-predictive objective as the core training signal for both encoders. It distills zero-shot policies end to end.

\paragraph{Sparse and structured representation learning.}
Sparsity-inducing penalties encourage weights or activations toward zero. We refer to \citet{hoefler2021sparsity} for a survey. A useful distinction is between \emph{model sparsity}, which permanently removes structural elements, and \emph{ephemeral sparsity}, which sparsifies activations per-example without altering the architecture. Dropout \cite{srivastava2014dropout} applies ephemeral sparsity to hidden units for regularization. $k$-sparse autoencoders \cite{makhzani2013k} retain only the top-$k$ activations per example for unsupervised feature learning. Closer to our motivation, \citet{liu2019utility} show that ephemeral sparsity in the last hidden layer of a control network improves value-function learning on classic control tasks. It reduces activation overlap between unrelated states. This mitigates interference during bootstrapping. It is evidence that sparsity can directly aid the kind of TD-based learning our method builds on.

\section{Background}
\label{sec:background}

\subsection{Problem Setup}

We consider a Markov Decision Process (MDP)
$\mathcal{M} = (\mathcal{S}, \mathcal{A}, P, R, \gamma)$, where $\mathcal{S}$ is the state space, $\mathcal{A}$ is the action
space,  $P(s' \mid s, a)$  is the  transition kernel, giving the probability measure over the next state $s'$ when action $a$ is taken in state $s$, $R(s, a)$ is the reward function, and  $\gamma \in [0,1]$ is the discount factor, giving down-weights rewards received further in the future~\citep{sutton1998reinforcement,touati2021learning}.
A \emph{policy} $\pi : \mathcal{S} \to \mathrm{Prob}(\mathcal{A})$ maps each
state to a probability distribution over actions \citep{bagatella2025td,touati2021learning}.

In \emph{zero-shot offline RL}~\citep{touati2021learning,park2025ogbench},
a single pre-trained policy must generalize to downstream reward without further environment interaction.
The agent pre-trains on a static, reward-free dataset $\mathcal{D} = \{(s, a, s')\}$ of transitions, where $s$ is the current state, $a$ is the action taken, and $s'$ is the resulting next state~\citep{bagatella2025td}.
At evaluation time it receives a small set of labeled transitions $\mathcal{D}_{\mathrm{rwd}} = \{(s, r)\}$, where $r \in \mathbf{R}$ is the scalar reward observed at state $s$~\citep{bagatella2025td}, and must produce a near-optimal policy
with no further updates to the pre-trained representation.

\subsection{Temporal Difference Joint-Embedding Predictive Architecture}
\label{sec:tdjepa}

TD-JEPA~\citep{bagatella2025td} trains a \emph{state encoder}
$\phi$ and a \emph{task encoder} $\psi$, together with two
predictors $T_\phi$ and $T_\psi$, via temporal-difference
latent-predictive losses on offline transitions
$(s,a,s') \in \mathcal{D}$.
The \emph{state encoder} $\phi : \mathcal{S} \to \mathbf{R}^{d_{\phi}}$
compresses each state to a $d_\phi$-dimensional latent vector.
The \emph{task encoder} $\psi : \mathcal{S} \to \mathbf{R}^{d_{\psi}}$ maps each state to a $d_\psi$-dimensional embedding used to define the space of linear rewards.
The predictor
$T_{\phi} : \mathbf{R}^{d_{\phi}} \times \mathcal{A} \times Z \to \mathbf{R}^{d_{\psi}}$
takes a state latent, an action, and a task vector and predicts in the
$d_\psi$-dimensional task-encoder space.
The predictor
$T_{\psi} : \mathbf{R}^{d_{\psi}} \times \mathcal{A} \times Z \to \mathbf{R}^{d_{\phi}}$
takes a task latent, an action, and a task vector and predicts in the
$d_\phi$-dimensional state-encoder space.

Rewards are assumed linear in the task encoding:
\begin{equation}
  r(s) \approx \psi(s)^{\top} z_r, \quad z_r \in \mathbf{R}^{d_\psi},
  \label{eq:reward_linear}
\end{equation}
where $z_r$ is a \emph{task vector} identifying a specific reward
function~\citep{barreto2017successor,bagatella2025td}. Task vectors
are constrained to a sphere of radius $\sqrt{d_\psi}$ via
\begin{equation}
  \mathrm{project\_z}(z) = \sqrt{d_\psi}\;\frac{z}{\|z\|_2}.
  \label{eq:projz}
\end{equation}

During pre-training, task vectors $z$ are sampled from a mixture of the uniform distribution on this sphere and projected task encodings of dataset states, and the actor is conditioned on $z$; each $z$ thus indexes a policy $\pi_z(\cdot \mid s)$ that is distilled end-to-end. Throughout our analysis we write $V(s) := \psi(s)^{\top} z_r$ for the \emph{inferred reward estimate} induced by a recovered task vector $z_r$; it is a point estimate of the immediate reward $r(s)$ in Eq.~\eqref{eq:reward_linear}.

\paragraph{Zero-shot inference.}
Given $\mathcal{D}_{\mathrm{rwd}} = \{(s, r)\}$, the task vector is
recovered by~\citep{bagatella2025td}
\begin{align}
    &z_r
  = \mathrm{project\_z}\!\left(
      (\Psi^{\top}\Psi)^{-1}\Psi^{\top} r
    \right)\\
  &\Psi = [\psi(s_1)\;\cdots\;\psi(s_N)]^{\top}.
  \label{eq:inference}
\end{align}

Here $N = |\mathcal{D}_{\mathrm{rwd}}|$ is the number of state-reward pairs inside buffer $\mathcal{D}_{\mathrm{rwd}}$, $\Psi \in \mathbf{R}^{N \times d_\psi}$ is the feature matrix
stacking the task encodings of all $N$ labeled states, and
$r \in \mathbf{R}^N$ collects their observed rewards.
The inner expression $(\Psi^{\top}\Psi)^{-1}\Psi^{\top} r$ is the
ordinary least-squares fit to $\Psi z \approx r$,
minimising $\sum_{(s,r)\in\mathcal{D}_{\mathrm{rwd}}}(r - \psi(s)^{\top} z)^2$
over $z \in \mathbf{R}^{d_\psi}$.
The result is then projected to the sphere via $\mathrm{project\_z}$ to
yield $z_r$.
No further updates to $\phi$ or $\psi$ are required at test time.

\paragraph{Reward Encoding Is Difficult In Zero-Shot RL.}
Zero-shot inference (Eq.~\eqref{eq:inference}) solves
$\Psi z \approx r$ by least squares.
In sparse-reward environments, the reward signal is
nonzero for fewer than $1\%$ of transitions, making $r$ almost
entirely zero.

With a dense $\psi$, only a handful of its $d_\psi$ output coordinates
may actually track reward-relevant structure; the remaining
coordinates vary across states for reasons unrelated to reward.
This is expected of $\psi$ as trained: the TD-JEPA objective optimises $\psi$ to serve every policy
$\pi_z$, $z \in Z$, and to supply targets for the multi-step
predictor $T_\psi$, not to discriminate the single reward $r$ given
at inference time.
Individual coordinates may therefore encode nuisance state variables,
structure relevant to other tasks in $Z$, or transient dynamics needed
for multi-step prediction -- all useful for $\psi$'s training
objective, but uninformative about $r$.

Because the least-squares fit in Eq.~\eqref{eq:inference} assigns a
weight to every coordinate of $\psi$ regardless of relevance, the
reward-irrelevant coordinates pick up spurious correlations
with the few nonzero entries of $r$ over the finite sample
$\mathcal{D}_{\mathrm{rwd}}$.
The resulting $z_r$ is nonzero across all $d_\psi$ coordinates, but
mass on reward-irrelevant coordinates adds noise
to every downstream prediction $\psi(s)^{\top}z_r$
(Eq.~\eqref{eq:reward_linear}), reducing the contrast between
goal-region states and the background and flattening the predicted
reward landscape relative to what a more concentrated $z_r$ would produce.

\begin{algorithm}[t]
\caption{Update ($\phi$: state encoder, $\psi$: task encoder, $T_\phi$: state predictor, $T_\psi$: task predictor) in \textsc{Sparse TD-JEPA}}
\label{alg:sparse-tdjepa}
\begin{algorithmic}
\setlength{\itemsep}{6pt}
\STATE $\{(s_i,a_i,s_i')\}_{i=1}^B \leftarrow \mathrm{SAMPLE}(\mathcal{D})$; \ $z_i \leftarrow \mathrm{SAMPLE\_Z}()$
\STATE $\mathcal{L} \leftarrow \mathcal{L}_\phi + \mathcal{L}_\psi + \lambda_\phi \mathcal{L}_{\mathrm{orth}}(\phi) + \lambda_\psi \mathcal{L}_{\mathrm{orth}}(\psi)$
\STATE \textcolor{blue}{$\mathcal{L} \leftarrow \mathcal{L} + \lambda_{\mathrm{sp}} \cdot \tfrac{1}{B}\sum_i \|\mathrm{project\_z}(\psi(s_i))\|_1$} \hfill \textcolor{blue}{$\triangleright$ \textbf{only one} line needs to be added.}
\STATE $\phi,\psi,T_\phi,T_\psi \leftarrow \mathrm{ADAM\_STEP}(\phi,\psi,T_\phi,T_\psi,\ \nabla \mathcal{L})$
\end{algorithmic}
\end{algorithm}

\section{Method: Sparse TD-JEPA}
\label{sec:method}

We propose to induce \emph{per-example sparsity} in the projected task embeddings $\mathrm{project\_z}(\psi(s))$ during training, encouraging each embedding to concentrate its mass on fewer strongly active dimensions.
The intended effect is that reward-predictive information occupies a smaller effective subspace of $\psi$'s output, so that least-squares inference recovers $z_r$ as a direction aligned with that subspace, producing a sharper, more concentrated inferred reward landscape with a stronger signal for the downstream policy to follow.
We note that Eq.~\eqref{eq:sparse} constrains each embedding individually and does not by itself force different states to share a support; whether a shared low-dimensional structure emerges is an empirical question, which we probe directly in the experiments (via the inferred reward landscape and its correlation with task progress) rather than assume.
Because the encoder output is constrained to the sphere of radius $\sqrt{d_\psi}$, the $\ell_1$ penalty cannot shrink the magnitude of the representation; on the sphere, $\|x\|_1$ is minimized exactly by coordinate-sparse directions, so the penalty purely reorients embeddings toward sparser directions rather than acting as a generic shrinkage regularizer.

\subsection{Sparsity Penalty}
\label{sec:sparse_loss}

We augment the TD-JEPA objective with an $\ell_1$ penalty on the
\emph{projected} task encoder output:
\begin{equation}
  \mathcal{L}_{\mathrm{sp}}
  = \frac{1}{B}\sum_{i=1}^{B}
    \bigl\|\,\mathrm{project\_z}(\psi(s_i))\,\bigr\|_1.
  \label{eq:sparse}
\end{equation}
Here $\|\cdot\|_1 = \sum_k |(\cdot)_k|$ denotes the sum of absolute
values, and the mean is taken over the $B$ states in the minibatch.
The full training objective is
\begin{equation}
  \mathcal{L}
  = \mathcal{L}_\phi + \mathcal{L}_\psi
  + \lambda_\phi\,\mathcal{L}_{\mathrm{orth}}(\phi)
  + \lambda_\psi\,\mathcal{L}_{\mathrm{orth}}(\psi)
  + \lambda_{\mathrm{sp}}\,\mathcal{L}_{\mathrm{sp}},
  \label{eq:total}
\end{equation}
where $\lambda_\phi, \lambda_\psi \geq 0$ are the orthonormality loss
coefficients fixed from the base TD-JEPA configuration and
$\lambda_{\mathrm{sp}} \geq 0$ is the sparsity-controlling coefficient. 
Dividing by $B$ mirrors the normalisation of $\mathcal{L}_\phi$ and
$\mathcal{L}_\psi$ (both use \texttt{.pow(2).sum(-1).mean()}),
making $\lambda_{\mathrm{sp}}$ a stable fraction of the main loss
across values of $B$ and $d_\psi$.
The implementation of Sparse TD-JEPA is extremely simple, requiring only one additional line of code added to the base update step, as shown in Algorithm~\ref{alg:sparse-tdjepa}.

\subsection{Theoretical Analysis}
\label{sec:theory}

Appendix~\ref{app:theory} formally characterizes what the penalty in Eq.~\eqref{eq:sparse} optimizes and how this propagates to zero-shot inference. The central quantity of the analysis is the \emph{effective dimension} (numerical sparsity) of an embedding $x = \mathrm{project\_z}(\psi(s))$,
\begin{equation}
  d_{\mathrm{eff}}(x) \;:=\; \frac{\|x\|_1^2}{\|x\|_2^2} \;\in\; \bigl[1,\, \|x\|_0\bigr],
  \qquad
  \|x\|_1 \;=\; \sqrt{d_\psi\, d_{\mathrm{eff}}(x)},
  \label{eq:deff-main}
\end{equation}
where the identity on the right holds for every $x$ on the sphere of radius $\sqrt{d_\psi}$. The $\ell_1$ penalty is therefore an exact monotone measure of effective dimension and cannot shrink magnitude (Proposition~\ref{prop:identity}). Rotation-invariant regularizers such as $\ell_2$ penalties are constant on the sphere and have zero gradient (Theorem~\ref{thm:inert}). Among $\ell_p$ penalties, $p=1$ is the convex choice with the largest sparse-to-dense range (Proposition~\ref{prop:lp}). Finally, the inference error splits into an $N$-independent approximation term and a variance term $\sigma^2\mathrm{rank}(\Psi)/N$ (Proposition~\ref{prop:decomp}). The flat sparse-vs-dense gap across $N$ (Figure~\ref{fig:vary_n_r2}) therefore reflects a change in feature geometry.

\section{Experiments}
\label{sec:experiments}

We design our experiments to answer the following research questions:

\begin{itemize}
    \item RQ1: Does inducing sparsity in the task encoder improve zero-shot performance over the dense TD-JEPA baseline and established zero-shot RL methods, across locomotion, navigation, object manipulation, and puzzle-solving benchmarks?
    \item RQ2: Does sparsity produce a sharper, more spatially concentrated predicted reward signal near the goal, and does this explain the stuck-policy failure mode observed with the dense encoder?
    \item RQ3: How does the sharpness of the inferred reward signal depend on the size of the labeled-set $N$ used for zero-shot inference, and does our method retain its advantage across different inference budgets?
\end{itemize}

\begin{figure}[h]
\centering
\includegraphics[]{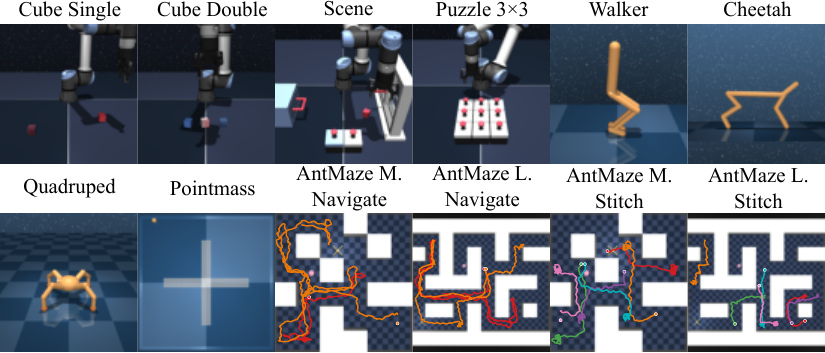} 
\vspace{-2em}
\caption{We test our method for zero-shot performance on various control, locomotive, and exploration tasks.}
\label{fig:environments_grid}
\end{figure}

\subsection{Environments and Datasets}

We evaluate our method on two comprehensive offline reinforcement learning benchmark suites: DeepMind Control Suite (DMC) \cite{tassa2018deepmind,tassa2020dmcontrol} and Offline Goal-Conditioned RL Benchmark (OGBench) \cite{park2025ogbench}, both are popular in recent zero-shot RL literature \cite{lawson2025self,park2024foundation,jajoo2026regularized}.

All datasets are publicly available and are used as-is in our experiments.

\textbf{\textit{DMC}} is a MuJoCo environment with locomotion tasks spanning \textit{walker}, \textit{cheetah}, \textit{quadruped}, and \textit{pointmass} domains. 
We evaluate our method on all four domains across 20 tasks, each trained on 3 million
environment steps of interaction. 
For training, we use the offline data set constructed with the EXploratory data for Offline Reinforcement Learning (ExoRL) protocol \cite{yarats2022don}. 
For zero-shot evaluation, we execute 10 independent rollouts per task and report the average return.

\textit{\textbf{OGBench}} is a MuJoCo RL benchmark with various locomotion, navigation, object manipulation, robotic arm manipulation, and puzzle-solving tasks.
We evaluate our method on nine OGBench domains, each with 5 tasks, for a total of 45 different tasks.
We use the \textit{Antmaze} domain for navigation, goal-stitching, and exploration \cite{park2025ogbench}, the \textit{Scene} and \textit{Puzzle} domains for puzzle-solving locomotion, and the \textit{Cube Single} and \textit{Cube Double} domains for object manipulation tasks using robotic arm.
For each domain, our model is trained on 1 million environment steps of offline interaction.
For zero-shot evaluation, we execute 10 independent rollouts per task and report the average return.
Performance is measured as the success rate of the binary outcome in each rollout \cite{park2025ogbench}.

\subsection{Baselines}

We compare our method against eight zero-shot RL baselines: Laplacian \cite{wu2018laplacian}, ICVF \cite{ghosh2023reinforcement}, HILP \cite{park2024foundation}, FB \cite{touati2021learning}, RLDP \cite{jajoo2026regularized}, BYOL \cite{grill2020bootstrap}, BYOL-$\gamma$ \cite{lawson2025self}, and TD-JEPA \cite{bagatella2025td}. Each baseline is tuned over its own per-domain grid of learning rates and orthonormality (or goal-sampling) coefficients to ensure a fair comparison. ICVF, BYOL, BYOL-$\gamma$ are representation-learning objectives not originally proposed for zero-shot RL. We use the public codebase of \citet{bagatella2025td}, which contains the adaptation of these methods for the tasks in our experiments.

\subsection{Computing And Environment Infrastructure}

OGBench experiments were run on a single NVIDIA H200 NVL GPU (141GB HBM3e), with an Intel Xeon 6960P (Granite Rapids) CPU, 754GB system RAM, and Linux (kernel 5.14, RHEL 9-family). DMC experiments were run on a single NVIDIA H100 NVL GPU (94GB HBM3), with an AMD EPYC 9454 (48-Core) CPU, 594GB system RAM, and Linux (kernel 5.15, Ubuntu 22.04-family).
All experiments used Python 3.11 and PyTorch 2.9 (CUDA 12.8), with gymnasium 1.2, ogbench 1.2, MuJoCo 3.3, NumPy 2.3, pandas 2.3, and scikit-learn 1.9.

\section{Results}

\subsection{Zero-shot Performance}

Table~\ref{tab:main_results} reports zero-shot performance across all benchmarks.
Sparse TD-JEPA has a higher mean score than TD-JEPA in every domain. 
On antmaze-medium-navigate ($85.4 \pm 4.4$), it surpasses the previous best, HILP ($83.6 \pm 2.6$), and improves over the TD-JEPA baseline ($70.4 \pm 3.7$) by $15$ points.
On antmaze-large-navigate ($70.6 \pm 8.4$), it outperforms the next best method, ICVF ($58.4 \pm 1.9$), and improves over TD-JEPA ($57.2 \pm 4.3$) by $13.4$ points.
Gains are similarly consistent on the more challenging stitch tasks: on antmaze-medium-stitch ($79.6 \pm 8.3$ vs.\ $61.6 \pm 4.5$ for TD-JEPA) and antmaze-large-stitch ($55.6 \pm 5.5$ vs.\ $40.6 \pm 2.5$ for TD-JEPA), our method achieves the highest success rates, including over the strongest stitch baseline, FB.
On the DMC locomotion benchmark, Sparse TD-JEPA achieves the best performance on all four tasks.

\begin{table*}[t]
\centering
\resizebox{\textwidth}{!}{%
\begin{tabular}{lrrrrrrrrr}
\toprule
 & Laplacian & ICVF & HILP & FB & RLDP & BYOL & BYOL-$\gamma$ & TD-JEPA & \textbf{Ours} \\
\midrule
\textbf{DMC (avg)}
  & $591 \pm 11$ & $619 \pm 10$ & $620 \pm 8$  & $648 \pm 4$
  & $610 \pm 14$ & $619 \pm 11$ & $645 \pm 11$ & $661 \pm 6$ & $\mathbf{714 \pm 3}$ \\
\quad walker
  & $770 \pm 5$  & $727 \pm 16$ & $796 \pm 8$  & $\mathbf{812 \pm 6}$
  & $724 \pm 18$ & $747 \pm 11$ & $786 \pm 10$ & $785 \pm 7$ & $\mathbf{814 \pm 8}$ \\
\quad cheetah
  & $615 \pm 19$ & $606 \pm 17$ & $618 \pm 6$  & $673 \pm 5$
  & $576 \pm 45$ & $623 \pm 24$ & $647 \pm 9$  & $689 \pm 7$ & $\mathbf{720 \pm 6}$ \\
\quad quadruped
  & $635 \pm 39$ & $709 \pm 14$ & $695 \pm 11$ & $596 \pm 9$
  & $665 \pm 14$ & $612 \pm 28$ & $683 \pm 26$ & $691 \pm 5$ & $\mathbf{801 \pm 10}$ \\
\quad pointmass
  & $345 \pm 22$ & $436 \pm 11$ & $371 \pm 37$ & $\mathbf{513 \pm 20}$
  & $\mathbf{476 \pm 39}$ & $\mathbf{493 \pm 41}$ & $465 \pm 18$ & $\mathbf{479 \pm 24}$ & $\mathbf{523  \pm 40}$ \\
\midrule
\textbf{OGBench (avg)}
  & $14.8 \pm 1.3$ & $30.9 \pm 0.6$ & $38.0 \pm 1.1$ & $39.0 \pm 0.7$
  & $27.1 \pm 0.8$ & $26.4 \pm 0.8$ & $30.4 \pm 0.9$ & $38.0 \pm 0.8$ & $\mathbf{48.4 \pm 2.3}$ \\
\quad antmaze-mn
  & $50.0 \pm 4.9$ & $\mathbf{79.8 \pm 2.6}$ & $\mathbf{83.6 \pm 2.6}$ & $73.0 \pm 2.7$
  & $74.6 \pm 4.2$ & $58.4 \pm 2.0$ & $51.4 \pm 1.6$ & $70.4 \pm 3.7$ & $\mathbf{85.4 \pm 4.4}$ \\ 
\quad antmaze-ln
  & $21.6 \pm 3.9$ & $58.4 \pm 1.9$ & $52.6 \pm 3.9$ & $36.8 \pm 4.3$
  & $36.4 \pm 4.7$ & $26.6 \pm 3.0$ & $21.8 \pm 3.6$ & $57.2 \pm 4.3$ & $\mathbf{70.6 \pm 8.4}$ \\
\quad antmaze-ms
  & $21.4 \pm 4.3$ & $39.0 \pm 3.3$ & $50.6 \pm 2.5$ & $\mathbf{70.4 \pm 4.0}$
  & $58.4 \pm 3.3$ & $60.6 \pm 5.1$ & $45.6 \pm 2.8$ & $61.6 \pm 4.5$ & $\mathbf{79.6 \pm 8.3}$ \\
\quad antmaze-ls
  & $11.8 \pm 1.5$ & $13.2 \pm 1.6$ & $12.2 \pm 1.8$ & $\mathbf{49.8 \pm 5.6}$
  & $19.6 \pm 2.7$ & $25.8 \pm 4.3$ & $20.2 \pm 1.8$ & $40.6 \pm 2.5$ & $\mathbf{55.6 \pm 5.5}$ \\
\quad antmaze-me
  & $0.8 \pm 0.6$  & $0.0 \pm 0.0$  & $2.0 \pm 0.8$  & $\mathbf{51.6 \pm 2.7}$
  & $4.8 \pm 2.4$  & $11.4 \pm 2.3$ & $19.6 \pm 2.5$ & $20.2 \pm 2.4$ & $26.6 \pm 3.9$ \\  
\quad scene
  & $7.8 \pm 1.3$  & $\mathbf{45.4 \pm 2.3}$ & $\mathbf{43.8 \pm 1.9}$ & $12.8 \pm 1.6$
  & $11.6 \pm 1.6$ & $15.4 \pm 1.4$ & $14.4 \pm 2.3$ & $38.4 \pm 1.4$ & $\mathbf{43.4 \pm 2.7}$ \\
\quad puzzle-3x3
  & $2.8 \pm 0.7$  & $16.6 \pm 0.7$ & $2.8 \pm 0.7$  & $4.8 \pm 0.7$
  & $14.6 \pm 0.9$ & $13.2 \pm 1.9$ & $18.8 \pm 0.4$ & $15.6 \pm 1.1$ & $\mathbf{22.0 \pm 1.8}$ \\
\quad cube-single
  & $15.1 \pm 1.5$ & $20.4 \pm 1.9$ & $\mathbf{74.2 \pm 3.5}$ & $49.6 \pm 3.8$
  & $19.8 \pm 2.4$ & $22.0 \pm 3.2$ & $\mathbf{79.4 \pm 2.8}$ & $34.2 \pm 2.9$ & $46.0 \pm 1.4$ \\
\quad cube-double
  & $2.0 \pm 0.4$  & $5.0 \pm 0.8$  & $\mathbf{20.0 \pm 2.7}$ & $2.6 \pm 0.4$
  & $3.8 \pm 0.8$  & $4.4 \pm 0.7$  & $2.6 \pm 0.7$  & $3.6 \pm 0.8$ & $6.4 \pm 1.6$ \\
\bottomrule
\end{tabular}%
}
\vspace{-0.5em}
\caption{Performance of zero-shot algorithms for DMC (reward) and OGBench (success rate) with
proprioception inputs. We report means and standard errors across seeds; DMC values are rounded
to the nearest integer and OGBench values to one decimal place. Numbers are bold for top
algorithms if confidence intervals overlap.}
\label{tab:main_results}
\end{table*}

In the next experiment, we investigate why our method produces higher task performance than the dense baseline.

\begin{figure*}[h]
\centering
\includegraphics[]{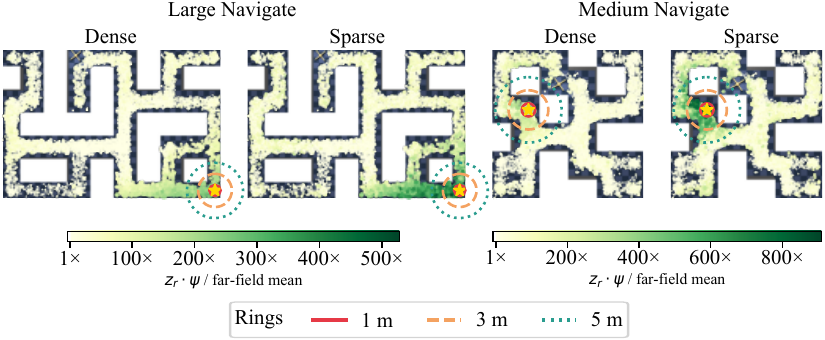}
\vspace{-2em}
\caption{Spatial reward landscape for task~3 on antmaze-large-navigate (\emph{left pair}) and antmaze-medium-navigate (\emph{right pair}).
Colour encodes $z_r^\top\psi(s)$ normalized by the model's own mean predicted reward for far-field states ($>$10\,m from goal).
Within each domain, Dense TD-JEPA (\emph{left}) and Sparse TD-JEPA (\emph{right}) share the same colour scale.
The star marks the goal; rings indicate 1\,m, 3\,m, and 5\,m distance.}
\label{fig:reward_landscape_ratio}
\end{figure*}

\textbf{Analysis Experiment 1: Distance--Reward Correlation.}
Figure~\ref{fig:reward_landscape_ratio} shows the spatial concentration of the predicted reward, $z_r^\top\psi(s)$ normalized by each model's own mean predicted reward for far-field states ($>$10\,m), for both encoders on the same task.
For the dense (baseline) encoder, the normalized reward stays elevated across nearly the whole goal room, giving little indication of the exact goal location once the agent is inside it. 
For the sparse encoder, the same signal is more informative: it peaks sharply at the goal itself and falls off visibly toward the room's edges. 

To verify this observation quantitatively, we measure the correlation betweem the Dense and Sparse model's learned value estimate $V(s) = \psi(s)^\top z_r$ against a distance-based task progress signal $x$: for the goal-conditioned OGBench domains, $x$ is the Euclidean distance from $s$ to the goal position; for the DMC domains, which have no spatial goal, $x$ is the environment's own instantaneous reward.
An example of such correlation is shown in Figure \ref{fig:scatterplot_full}.
For each of 20 independently resampled observation batches (shared between sparse and baseline), we fit $V(s) \sim x$ via linear regression and report the 5-fold cross-validated $R^2$. The statistical significance is determined through a Welch's $t$-test.

The results are shown in Figure~\ref{fig:boxplot_full}. 
Sparse encoder produced a statistically significant higher correlation between predicted reward and task completion distance than Dense encoder in all domains except for Large Navigate.

Testing whether the sparse encoder better represents other task dimensions besides distance-to-goal is left to future work.

\begin{figure}[h]
\centering
\includegraphics[]{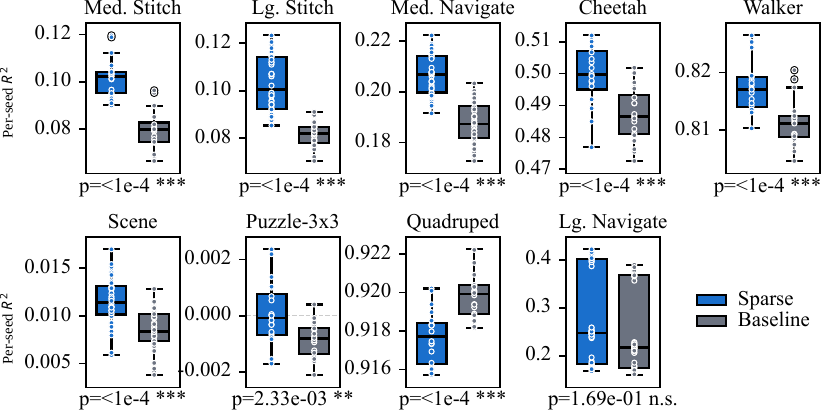} 
\vspace{-1.5em}
\caption{Correlation between each model's learned value estimate $V(s) = \psi(s) \cdot z_r$ and distance signal.}
\label{fig:boxplot_full}
\end{figure}

\textbf{Analysis Experiment 2: Effect of Labeled-Set Size $N$ on Reward Signal Quality.}
We characterise how well the inferred value estimate $V(s) = \psi(s)^\top z_r$ tracks the ground-truth signal (distance-to-goal for OGBench, reward for DMC), measured by $R^2$, as a function of the number of labeled transitions $N$ used to estimate $z_r$.
$N$ is the number of labeled (state, reward) transitions used to solve for the task vector $z_r$ via least-squares ($z_r = \mathrm{lstsq}(\psi(s), r)$, then projected onto the $\psi$-sphere).

\begin{figure}[h]
\centering
\includegraphics[]{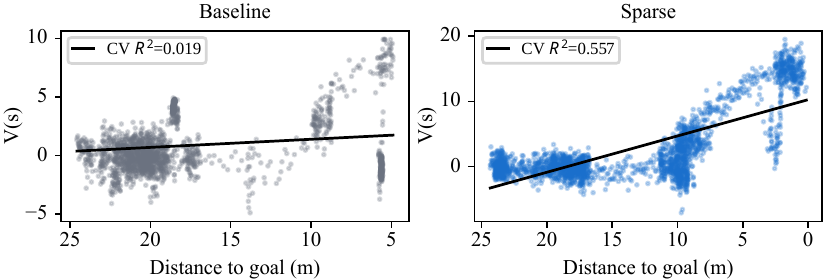}
\vspace{-2em}
\caption{Example of linear regression correlation of Dense (Baseline) and Sparse task encoders.}
\label{fig:scatterplot_full}
\end{figure}

$z_r$ has $\psi$-dim $=50$ free parameters, fit by plain unregularized least squares. Below $N \approx 50$ the system is underdetermined, and for $N$ in the low hundreds (just above 50) unregularized OLS sits at the interpolation threshold, where the fit is barely determined and hence maximally sensitive to noise in those specific points. To stay clear of this unstable regime, we sweep $N$ over log-decades starting at $N=10^4$ (well above $\psi$-dim) up to each domain's full training-dataset size; since that full-dataset size differs by domain, the swept range is not the same across domains.

For each $N$, we refit $z_r$ from 20 independently resampled batches of $N$ labeled transitions and, for each, compute the 5-fold cross-validated $R^2$ (degree 1) between $V(s)$ and the ground-truth signal, evaluated on a shared, fixed set of states so that sparse and baseline are compared on identical inputs at every $N$. 
This isolates the effect of labeled-set size from any confound in which states each policy visits.
We report the mean $R^2$ across each 20 independently resampled batches with their standard error (SEM), and flag any $N$ where the sparse-vs-baseline difference is statistically significant via Welch's $t$-test ($p<0.05$).

Results are shown in Figure~\ref{fig:vary_n_r2}.
The sparse encoder attains a significantly higher $R^2$ than the baseline across nearly the full swept range of $N$ in six of the nine domains (antmaze-ls, antmaze-mn, antmaze-ms, scene, cheetah, walker); antmaze-ln shows the same direction without reaching significance at any $N$; puzzle-3x3 shows no reliable signal for either model ($R^2 \approx 0$).
Within each domain the sparse-vs-baseline gap in $R^2$ stays roughly flat across $N$, indicating that the sparse encoder's advantage in how well $V(s)$ tracks the ground-truth signal is not merely a low-data artifact tied to the $N=10^4$ inference budget used in the main experiments (see the Experiments section), but persists as more labeled data becomes available.

\begin{figure}[h]
\centering
\includegraphics[width=\linewidth]{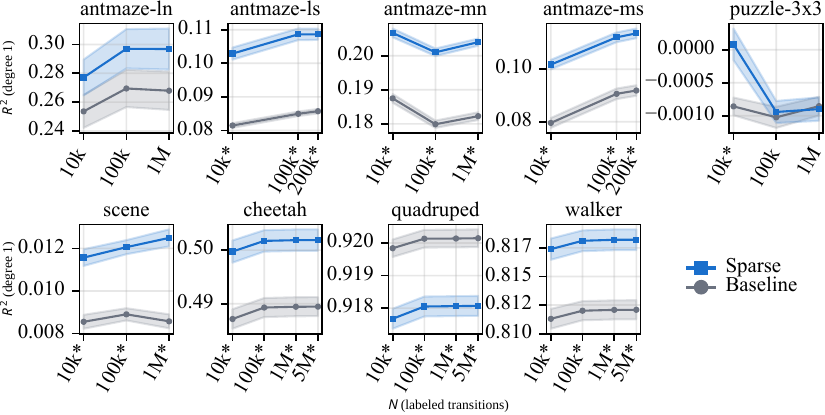}
\vspace{-2em}
\caption{Cross-validated $R^2$ between learned value estimate and the distance signal.}
\label{fig:vary_n_r2}
\end{figure}

\section{Conclusion}
\label{sec:conclusion}

We identified an underexplored failure mode of successor-feature-based zero-shot RL, dense task encoders yield a diffuse inferred reward signal with little contrast between goal-relevant states and the background, and addressed it with a simple one-line $\ell_1$ penalty on TD-JEPA's projected task embeddings, which provably controls only the effective dimensionality of the representation and leaves the closed-form inference procedure untouched. Sparse TD-JEPA improves over the dense baseline in almost every evaluated domain, and our analyses tie these gains to a sharper, more goal-concentrated inferred reward that persists across labeled-set sizes. Future work includes applying the sparsity penalty to other successor-feature-based zero-shot RL methods beyond TD-JEPA, and leveraging the resulting sparse, low-effective-dimension task embeddings for mechanistic interpretability of what individual latent dimensions encode.

\subsection*{AI use statement}

In this work, we used AI tools (LLM-based assistants) to assist with polishing the manuscript text. We have not used generative AI tools for research ideation or method design. The identification of the failure mode, the proposed sparsity penalty, the theoretical analysis, and all experimental conclusions are the authors' own. We have reviewed all AI-assisted work: AI-assisted text was checked against the underlying experimental results. We take responsibility for the final content of this work, including text, claims, and artifacts produced with the aid of generative AI.

\subsection*{Reproducibility statement}

We have taken several steps to make our results reproducible. All theoretical claims summarized in Section~\ref{sec:theory} are proved in Appendix~\ref{app:theory}, which states every assumption and includes complete proofs. Appendix~\ref{sec:seeds} documents the seeds and the result-aggregation protocol. Appendix~\ref{sec:baselines-grids} specifies all baselines and their hyperparameter grids, and Appendix~\ref{app:datasets} gives full details of the datasets, environments, and evaluation protocol. Compute and infrastructure details are provided in Section~\ref{sec:experiments}.

\bibliography{iclr2027_conference}
\bibliographystyle{iclr2027_conference}

\appendix

\section{Theoretical Analysis of the Sparsity Penalty}
\label{app:theory}

This appendix develops the formal analysis summarized in Section~\ref{sec:theory}. Throughout, we characterize (a)~the quantity that the penalty $\mathcal{L}_{\mathrm{sp}}$ of Eq.~\eqref{eq:sparse} actually optimizes under TD-JEPA's sphere constraint, (b)~the behavior of the broader family of norm-based regularizers under the same constraint, (c)~the optimization dynamics that the penalty induces on the sphere, and (d)~how the geometry of the task encoder propagates to the least-squares inference step of Eq.~\eqref{eq:inference}. The section closes with a systematic comparison of alternative regularizers in light of these results.

\subsection{Setup and Notation}
\label{app:theory-setup}

Throughout this appendix we abbreviate $d := d_\psi$ and write
\begin{equation}
  \mathcal{S}_d \;:=\; \bigl\{\, x \in \mathbf{R}^d \;:\; \|x\|_2 = \sqrt{d} \,\bigr\}
\end{equation}
for the sphere onto which task embeddings and task vectors are projected, and
\begin{equation}
  \Pi(y) \;:=\; \mathrm{project\_z}(y) \;=\; \sqrt{d}\,\frac{y}{\|y\|_2}, \qquad y \neq 0,
\end{equation}
for the projection map of Eq.~\eqref{eq:projz}. The penalized quantity in Eq.~\eqref{eq:sparse} is $x = \Pi(\psi(s)) \in \mathcal{S}_d$, so every statement below about $x \in \mathcal{S}_d$ applies verbatim to the embeddings seen by the penalty. For $x \in \mathbf{R}^d$ we write $\|x\|_p = (\sum_k |x_k|^p)^{1/p}$ for $p \geq 1$, $\mathrm{supp}(x) = \{k : x_k \neq 0\}$, and $\|x\|_0 = |\mathrm{supp}(x)|$. We call $x$ \emph{$1$-sparse} if $\|x\|_0 = 1$ and \emph{flat on its support} if $|x_k|$ is constant over $k \in \mathrm{supp}(x)$; a flat vector with full support ($\|x\|_0 = d$) is called \emph{flat}. An \emph{orthogonal transformation} is a matrix $U \in O(d)$, i.e., $U^\top U = I$.

Our analysis is phrased in terms of the following continuous notion of sparsity, standard in the sparse-recovery literature \citep{hoyer2004non,hurley2009comparing}.

\begin{definition}[Effective dimension / numerical sparsity]
\label{def:effdim}
For $x \in \mathbf{R}^d \setminus \{0\}$, the \emph{effective dimension} of $x$ is
\begin{equation}
  d_{\mathrm{eff}}(x) \;:=\; \frac{\|x\|_1^2}{\|x\|_2^2}.
\end{equation}
\end{definition}

The functional $d_{\mathrm{eff}}$ is scale-invariant ($d_{\mathrm{eff}}(cx) = d_{\mathrm{eff}}(x)$ for $c \neq 0$) and continuous on $\mathbf{R}^d \setminus \{0\}$; Proposition~\ref{prop:identity} shows $1 \leq d_{\mathrm{eff}}(x) \leq \|x\|_0$, so $d_{\mathrm{eff}}$ is a continuous surrogate for the (discontinuous) support size $\|x\|_0$, and $d_{\mathrm{eff}}(x) = k$ exactly on vectors that are flat on a support of size $k$. Intuitively, $d_{\mathrm{eff}}(x)$ counts the number of coordinates of $x$ that carry non-negligible mass.

\subsection{The Penalty Is an Exact Measure of Effective Dimensionality}
\label{app:theory-identity}

The first result makes precise, and strengthens, the informal claim of Section~\ref{sec:method}: under the sphere constraint the $\ell_1$ penalty coincides, up to a fixed strictly monotone transformation, with the effective dimension of the embedding. In particular the penalty carries no information about, and exerts no pressure on, the magnitude of the representation, because that magnitude is fixed by the constraint.

\begin{proposition}[Exact identity and extreme points]
\label{prop:identity}
Let $x \in \mathcal{S}_d$. Then:
\begin{enumerate}
    \item[(i)] $\displaystyle \|x\|_1 = \sqrt{d\, d_{\mathrm{eff}}(x)}$, and consequently the per-example penalty is a strictly increasing function of the effective dimension $d_{\mathrm{eff}}(x)$;
    \item[(ii)] $1 \leq d_{\mathrm{eff}}(x) \leq \|x\|_0 \leq d$, with $d_{\mathrm{eff}}(x) = \|x\|_0$ if and only if $x$ is flat on its support;
    \item[(iii)] $\sqrt{d} \leq \|x\|_1 \leq d$; the lower bound is attained exactly at the $2d$ signed axis vectors $x = \pm\sqrt{d}\,e_k$ (the $1$-sparse elements of $\mathcal{S}_d$), and the upper bound exactly at the $2^d$ flat vectors with $|x_k| = 1$ for all $k$ (the maximally dense elements).
\end{enumerate}
\end{proposition}

\begin{proof}
(i) is immediate from Definition~\ref{def:effdim} and $\|x\|_2^2 = d$: $d_{\mathrm{eff}}(x) = \|x\|_1^2/d$, hence $\|x\|_1 = \sqrt{d\,d_{\mathrm{eff}}(x)}$, and $t \mapsto \sqrt{d\,t}$ is strictly increasing.

(ii) Lower bound: $\|x\|_1^2 = \sum_k x_k^2 + \sum_{j \neq k} |x_j||x_k| \geq \|x\|_2^2$, so $d_{\mathrm{eff}}(x) \geq 1$, with equality if and only if all cross terms vanish, i.e., $x$ has at most one nonzero coordinate. Upper bound: by the Cauchy--Schwarz inequality applied over $\mathrm{supp}(x)$,
\begin{equation}
  \|x\|_1 = \sum_{k \in \mathrm{supp}(x)} 1 \cdot |x_k|
  \;\leq\; \sqrt{\|x\|_0}\, \Bigl(\sum_k x_k^2\Bigr)^{1/2}
  = \sqrt{\|x\|_0}\,\|x\|_2,
\end{equation}
so $d_{\mathrm{eff}}(x) \leq \|x\|_0$, with equality if and only if $(|x_k|)_{k \in \mathrm{supp}(x)}$ is proportional to the all-ones vector, i.e., $x$ is flat on its support.

(iii) follows by combining (i) with (ii): $\|x\|_1 = \sqrt{d\,d_{\mathrm{eff}}(x)} \in [\sqrt{d}, \sqrt{d \cdot d}] = [\sqrt{d}, d]$, and the equality characterizations transfer. For the lower extreme, $d_{\mathrm{eff}}(x) = 1$ forces $\|x\|_0 = 1$, and membership in $\mathcal{S}_d$ then forces $x = \pm\sqrt{d}\,e_k$. For the upper extreme, $d_{\mathrm{eff}}(x) = d$ forces $\|x\|_0 = d$ and flatness, and membership in $\mathcal{S}_d$ forces $|x_k| = 1$ for all $k$.
\end{proof}

\begin{corollary}[No shrinkage; the penalty acts on direction only]
\label{cor:noshrink}
For every state $s$, the penalized embedding satisfies $\|\Pi(\psi(s))\|_2 = \sqrt{d}$ identically, so the penalty cannot decrease (or increase) the norm of any representation. By Proposition~\ref{prop:identity}(i), minimizing $\mathcal{L}_{\mathrm{sp}} = \frac{1}{B}\sum_i \|\Pi(\psi(s_i))\|_1 = \frac{\sqrt{d}}{B}\sum_i \sqrt{d_{\mathrm{eff}}\bigl(\Pi(\psi(s_i))\bigr)}$ is equivalent to minimizing the batch average of $\sqrt{d_{\mathrm{eff}}(\cdot)}$; the sole quantity the penalty can move is the per-example effective dimension.
\end{corollary}

\begin{remark}
This is the formal counterpart of the ridge-versus-lasso distinction \citep{tibshirani1996regression} in a constrained setting: in unconstrained regression, an $\ell_1$ penalty both shrinks and sparsifies, and its effects are entangled; on the sphere the shrinkage component is annihilated by the constraint, leaving sparsification as the only degree of freedom. The training signal added by Algorithm~\ref{alg:sparse-tdjepa} is therefore a pure effective-dimensionality pressure.
\end{remark}

\subsection{Behavior of Norm-Based Regularizers Under the Sphere Constraint}
\label{app:theory-inert}

We next characterize how the broader design space of regularizers interacts with the sphere constraint. The organizing concept is rotation invariance.

\begin{definition}[Rotation-invariant regularizer]
\label{def:rotinv}
A function $\Omega : \mathbf{R}^d \to \mathbf{R}$ is \emph{rotation-invariant} if $\Omega(Ux) = \Omega(x)$ for every $x \in \mathbf{R}^d$ and every $U \in O(d)$.
\end{definition}

Examples include $\Omega(x) = \|x\|_2$, $\Omega(x) = \|x\|_2^2$, any composition $\Omega(x) = h(\|x\|_2)$ with $h : \mathbf{R}_{\geq 0} \to \mathbf{R}$, and smoothed variants such as $\Omega(x) = \mathbb{E}_{\xi \sim \mathcal{N}(0, \tau^2 I)}\, h(\|x + \xi\|_2)$, since an isotropic Gaussian is itself rotation-invariant.

\begin{theorem}[Inertness of rotation-invariant regularizers]
\label{thm:inert}
Let $\Omega : \mathbf{R}^d \to \mathbf{R}$ be rotation-invariant. Then:
\begin{enumerate}
    \item[(i)] $\Omega$ is constant on $\mathcal{S}_d$; call this constant $c_\Omega$;
    \item[(ii)] the composition $\Omega \circ \Pi$ is constant on $\mathbf{R}^d \setminus \{0\}$, and its gradient is identically zero wherever it is differentiable;
    \item[(iii)] for every coefficient $\lambda \geq 0$, the objective $\mathcal{L} + \lambda\, \frac{1}{B}\sum_i \Omega\bigl(\Pi(\psi(s_i))\bigr)$ has the same gradients with respect to all parameters as $\mathcal{L}$, at every parameter vector; hence it has the same stationary points, and gradient-based training (including SGD and Adam, for any step-size schedule and any realization of minibatch sampling) produces identical parameter trajectories.
\end{enumerate}
\end{theorem}

\begin{proof}
(i) The orthogonal group acts transitively on each centered sphere: for any $x, y$ with $\|x\|_2 = \|y\|_2$ there exists $U \in O(d)$ with $Ux = y$ (complete $x/\|x\|_2$ and $y/\|y\|_2$ to orthonormal bases and map one to the other). Hence $\Omega(y) = \Omega(Ux) = \Omega(x)$ for all $x, y \in \mathcal{S}_d$.

(ii) For every $y \neq 0$, $\Pi(y) \in \mathcal{S}_d$, so $\Omega(\Pi(y)) = c_\Omega$ by (i). A constant function has zero gradient wherever differentiable.

(iii) By (ii) the added term equals the constant $\lambda c_\Omega$ as a function of the parameters (through any parameterization of $\psi$), so its contribution to every parameter gradient is zero; the chain rule through $\Pi$ and $\psi$ transmits a zero upstream gradient. Identical gradients at every point imply identical update sequences for any first-order method fed the same minibatch sequence, hence identical trajectories and identical sets of stationary points.
\end{proof}

\begin{corollary}[Instances]
\label{cor:instances}
Under the sphere constraint of TD-JEPA, an $\ell_2$ penalty $\|\Pi(\psi(s))\|_2$, a squared penalty $\|\Pi(\psi(s))\|_2^2$, any penalty of the form $h(\|\Pi(\psi(s))\|_2)$, and any isotropically smoothed norm penalty are all inert in the sense of Theorem~\ref{thm:inert}: adding them changes neither the training trajectory nor the final encoder, for any coefficient value.
\end{corollary}

Two widely used regularizers act on parameters or activations rather than directly on the projected embedding; the following remarks characterize their behavior in this architecture.

\begin{remark}[Weight decay]
\label{rem:wd}
The projection is scale-invariant: $\Pi(cy) = \Pi(y)$ for all $c > 0$. Consequently, if the final linear layer of $\psi$ is scaled by $c > 0$, every projected embedding $\Pi(\psi(s))$, and hence every quantity entering the reward model of Eq.~\eqref{eq:reward_linear} and the inference step of Eq.~\eqref{eq:inference}, is unchanged. Along such scaling rays the data-fitting terms of the objective are constant while the weight-decay term strictly decreases, so weight decay on the output layer moves parameters along function-preserving rays: it controls the parameter norm (equivalently, the effective learning rate of normalized layers; cf.\ \citealp{vanlaarhoven2017l2}) without expressing any preference among the distinct projected representations that training could select. In particular it exerts no pressure on $d_{\mathrm{eff}}(\Pi(\psi(s)))$.
\end{remark}

\begin{remark}[Dropout]
\label{rem:dropout}
Dropout applied to the task-encoder output \citep{srivastava2014dropout} multiplies coordinates by independent Bernoulli variables. In expectation, dropout on a representation feeding a (locally) linear readout is equivalent, to second order, to a coordinate-wise \emph{quadratic} penalty weighted by the per-coordinate second moments $\mathbb{E}[x_k^2]$ \citep{wager2013dropout}. Two properties of this induced penalty are relevant here. First, it is quadratic, and quadratic penalties lack the polyhedral vertex geometry through which $\ell_1$ produces (approximately) sparse solutions \citep{tibshirani1996regression}. Second, TD-JEPA's orthonormality loss drives the second moments toward isotropy, $\mathbb{E}[\psi(s)\psi(s)^\top] \approx I$; in that regime the induced penalty degenerates toward a function of the $\ell_2$ norm alone, which is inert on the sphere by Corollary~\ref{cor:instances}. Dropout therefore injects noise, but does not supply a per-example effective-dimensionality pressure.
\end{remark}

\subsection{The $\ell_p$ Family: Which Exponents Favor Sparse Directions}
\label{app:theory-lp}

Theorem~\ref{thm:inert} shows that a regularizer must be coordinate-anisotropic to have any effect at all under the sphere constraint. The canonical anisotropic family is the $\ell_p$ family. The next proposition determines, for every exponent, which directions on the sphere the penalty favors, and locates $\ell_1$ within the family.

\begin{proposition}[$\ell_p$ penalties on the sphere]
\label{prop:lp}
Let $x \in \mathcal{S}_d$.
\begin{enumerate}
    \item[(i)] For $1 \leq p < 2$:\; $d^{1/2} \leq \|x\|_p \leq d^{1/p}$, with the minimum attained exactly at $1$-sparse vectors and the maximum exactly at flat vectors. Minimizing $\|x\|_p$ over $\mathcal{S}_d$ therefore favors sparse directions.
    \item[(ii)] For $p = 2$:\; $\|x\|_2 \equiv \sqrt{d}$ on $\mathcal{S}_d$; the penalty is constant (Theorem~\ref{thm:inert}).
    \item[(iii)] For $2 < p \leq \infty$:\; $d^{1/p} \leq \|x\|_p \leq d^{1/2}$, with the minimum attained exactly at flat vectors and the maximum exactly at $1$-sparse vectors. Minimizing $\|x\|_p$ over $\mathcal{S}_d$ therefore favors \emph{dense} directions.
    \item[(iv)] Within the convex range $p \geq 1$, the ratio between the penalty's value at its densest and sparsest extremes on $\mathcal{S}_d$ equals $d^{|1/p - 1/2|}$, which is maximized at $p = 1$, where it equals $\sqrt{d}$. In this sense $\ell_1$ is the convex member of the family with the largest dynamic range between sparse and dense embeddings.
\end{enumerate}
\end{proposition}

\begin{proof}
We use two standard inequalities for $0 < a \leq b \leq \infty$ and $x \in \mathbf{R}^d$:
\begin{equation}
  \text{(N)}\;\; \|x\|_b \leq \|x\|_a,
  \qquad
  \text{(H)}\;\; \|x\|_a \leq d^{\,1/a - 1/b}\, \|x\|_b .
  \label{eq:norm-ineqs}
\end{equation}
(N) follows by normalizing $\|x\|_a = 1$, so that $|x_k| \leq 1$ and $|x_k|^b \leq |x_k|^a$ for each $k$; summing gives $\|x\|_b^b \leq 1$. Equality in (N) holds if and only if $\|x\|_0 \leq 1$: with two or more nonzero entries, each normalized entry is strictly below $1$, so $|x_k|^b < |x_k|^a$ for those entries. (H) is H\"older's inequality applied to $\sum_k |x_k|^a \cdot 1$ with exponent pair $(b/a,\, (1 - a/b)^{-1})$:
\begin{equation}
  \|x\|_a^a = \sum_k |x_k|^a \cdot 1
  \leq \Bigl(\sum_k |x_k|^{b}\Bigr)^{a/b} d^{\,1 - a/b}
  = \|x\|_b^a\, d^{\,1 - a/b},
\end{equation}
with equality if and only if $(|x_k|^a)_k$ is proportional to the all-ones vector, i.e., $|x|$ is flat.

(i) Apply (N) with $(a, b) = (p, 2)$: $\|x\|_p \geq \|x\|_2 = d^{1/2}$, equality iff $x$ is $1$-sparse. Apply (H) with $(a, b) = (p, 2)$: $\|x\|_p \leq d^{1/p - 1/2}\|x\|_2 = d^{1/p}$, equality iff $x$ is flat.

(iii) Apply (N) with $(a, b) = (2, p)$: $\|x\|_p \leq \|x\|_2 = d^{1/2}$, equality iff $x$ is $1$-sparse. Apply (H) with $(a, b) = (2, p)$: $d^{1/2} = \|x\|_2 \leq d^{1/2 - 1/p}\|x\|_p$, i.e., $\|x\|_p \geq d^{1/p}$, equality iff $x$ is flat. (For $p = \infty$ the same conclusions follow by taking limits, with $\|x\|_\infty \in [1, \sqrt d\,]$.)

(ii) is immediate, and (iv) follows by comparing the extreme values from (i) and (iii): the ratio is $d^{1/p}/d^{1/2} = d^{1/p - 1/2}$ for $p < 2$ and $d^{1/2}/d^{1/p} = d^{1/2 - 1/p}$ for $p > 2$; over $p \geq 1$ the exponent $|1/p - 1/2| \leq 1/2$ is maximized at $p = 1$.
\end{proof}

\begin{remark}[Boundary members of the family]
\label{rem:l0}
For $0 < p < 1$ the map $x \mapsto \sum_k |x_k|^p$ favors sparse directions even more strongly, but is non-convex with unbounded gradients at the coordinate hyperplanes, which is problematic for stochastic first-order training; the limit $p \to 0$ recovers the support size $\|x\|_0$, which is piecewise constant and provides no gradient signal at all. The exponent $p = 1$ is the classical convex surrogate for $\|x\|_0$ used throughout sparse recovery \citep{tibshirani1996regression,donoho2006compressed}, and by Proposition~\ref{prop:lp}(iv) it is also the convex member with the largest sparse-to-dense dynamic range on the sphere. Proposition~\ref{prop:lp}(iii) additionally yields a falsifiable control: a penalty with $p > 2$ (or, equivalently, rewarding large $\|x\|_p$ for $p < 2$) is predicted to push embeddings toward the flat, maximally dense extreme, the opposite intervention to Eq.~\eqref{eq:sparse}.
\end{remark}

\subsection{Optimization Geometry of the Penalty on the Sphere}
\label{app:theory-dynamics}

We now describe the dynamics that the penalty induces. Since embeddings live on $\mathcal{S}_d$, the relevant object is the Riemannian gradient of $\|\cdot\|_1$ on the sphere, i.e., the tangential component of the Euclidean subgradient.

\begin{proposition}[Gradient flow of the penalty: mass transport with a participation threshold]
\label{prop:flow}
Let $x \in \mathcal{S}_d$ have no zero coordinate, and let $\sigma = \sign(x) \in \{\pm 1\}^d$. The Riemannian gradient of $f(x) = \|x\|_1$ on $\mathcal{S}_d$ at $x$ is
\begin{equation}
  \operatorname{grad} f(x) \;=\; \sigma \;-\; \frac{\|x\|_1}{d}\, x,
  \label{eq:riem-grad}
\end{equation}
and the penalty-only descent flow $\dot{x} = -\operatorname{grad} f(x)$ preserves $\mathcal{S}_d$ and satisfies, for every coordinate $k$,
\begin{equation}
  \frac{\mathrm{d}}{\mathrm{d}t}\, |x_k|
  \;=\; \frac{\|x\|_1}{d}\Bigl( |x_k| - \tau(x) \Bigr),
  \qquad
  \tau(x) := \frac{d}{\|x\|_1} \in [1, \sqrt{d}\,].
  \label{eq:threshold}
\end{equation}
That is, at fixed total $\ell_2$ mass, the penalty continuously transports magnitude from coordinates below the state-dependent participation threshold $\tau(x)$ to coordinates above it: strongly active coordinates grow and weakly active coordinates decay.
\end{proposition}

\begin{proof}
On the open set where no coordinate vanishes, $f(x) = \sigma^\top x$ is smooth with Euclidean gradient $\sigma$. The tangent space of $\mathcal{S}_d$ at $x$ is $\{v : x^\top v = 0\}$, with orthogonal projector $P_x = I - xx^\top/\|x\|_2^2 = I - xx^\top/d$. Hence
\begin{equation}
\operatorname{grad} f(x) = P_x \sigma = \sigma - \frac{\sigma^\top x}{d}\,x = \sigma - \frac{\|x\|_1}{d}\,x,
\end{equation}
using $\sigma^\top x = \|x\|_1$. The flow preserves the sphere because $\frac{\mathrm{d}}{\mathrm{d}t}\|x\|_2^2 = 2x^\top \dot{x} = -2\bigl(x^\top\sigma - \frac{\|x\|_1}{d}\|x\|_2^2\bigr) = -2(\|x\|_1 - \|x\|_1) = 0$. For each coordinate, $\frac{\mathrm{d}}{\mathrm{d}t}|x_k| = \sigma_k \dot{x}_k = -\sigma_k\bigl(\sigma_k - \frac{\|x\|_1}{d}x_k\bigr) = \frac{\|x\|_1}{d}|x_k| - 1$, which equals the right-hand side of Eq.~\eqref{eq:threshold}. The range of $\tau$ follows from Proposition~\ref{prop:identity}(iii): $\|x\|_1 \in [\sqrt d, d]$ implies $\tau(x) = d/\|x\|_1 \in [1, \sqrt d\,]$.
\end{proof}

\begin{proposition}[Critical points form a discrete effective-dimension ladder]
\label{prop:ladder}
Consider $f(x) = \|x\|_1$ restricted to $\mathcal{S}_d$, in the sense of Clarke-stationarity (there exist $g \in \partial\|\cdot\|_1(x)$ and $\mu \in \mathbf{R}$ with $g = \mu x$). Then:
\begin{enumerate}
    \item[(i)] the stationary points are exactly the vectors that are flat on their support: for each support size $k \in \{1, \dots, d\}$ these are the points with $|x_j| = \sqrt{d/k}$ for $j$ in the support, and the corresponding penalty value is $f(x) = \sqrt{kd}$, strictly increasing in $k$;
    \item[(ii)] for every $k \geq 2$, no stationary point with support size $k$ is a local minimizer: within the great circle obtained by rotating any two supported coordinates, the flat point is a strict local maximizer of $f$;
    \item[(iii)] the global minimizers are exactly the $2d$ signed axis vectors $\pm\sqrt{d}\,e_k$, with value $\sqrt{d}$.
\end{enumerate}
\end{proposition}

\begin{proof}
(i) The subdifferential of $\|\cdot\|_1$ at $x$ is $\partial\|\cdot\|_1(x) = \{g : g_j = \sign(x_j) \text{ if } x_j \neq 0,\; g_j \in [-1,1] \text{ otherwise}\}$. Stationarity requires $g = \mu x$ for some such $g$. On the support this reads $\sign(x_j) = \mu x_j$, so $\mu > 0$ and $|x_j| = 1/\mu$ is constant across the support; off the support the condition $g_j = \mu \cdot 0 = 0 \in [-1, 1]$ is always satisfiable. Thus stationary points are precisely the flat-on-support vectors. If the support has size $k$, the constraint $\|x\|_2^2 = k/\mu^2 = d$ gives $1/\mu = \sqrt{d/k}$ and $f(x) = k\sqrt{d/k} = \sqrt{kd}$.

(ii) Let $x$ be stationary with support size $k \geq 2$ and pick supported coordinates $i \neq j$, so $|x_i| = |x_j| = c := \sqrt{d/k}$. Consider the curve on $\mathcal{S}_d$ that rotates the pair while fixing all other coordinates:
\begin{equation}
x_i(\theta) = \sqrt{2}\,c\,\sigma_i \cos\theta, \qquad x_j(\theta) = \sqrt{2}\,c\,\sigma_j \sin\theta, \qquad \theta \in (0, \tfrac{\pi}{2}),
\end{equation}
which preserves $x_i^2 + x_j^2 = 2c^2$ and hence membership in $\mathcal{S}_d$, and passes through $x$ at $\theta = \pi/4$. Along this curve, $f(\theta) = C + \sqrt{2}\,c\,(\cos\theta + \sin\theta) = C + 2c \sin(\theta + \pi/4)$ with $C$ collecting the fixed coordinates; this is strictly maximized at $\theta = \pi/4$. Hence $f$ strictly decreases in both directions along the curve, and $x$ is not a local minimizer.

(iii) By Proposition~\ref{prop:identity}(iii), $f \geq \sqrt{d}$ on $\mathcal{S}_d$ with equality exactly at the signed axis vectors, which are the $k = 1$ rung of the ladder in (i).
\end{proof}

\begin{remark}[Interpretation]
\label{rem:dynamics-interp}
Propositions~\ref{prop:flow} and~\ref{prop:ladder} together describe the penalty's contribution to training as a descent process on a ladder of configurations indexed by effective dimension: penalty values $\sqrt{kd}$ at the flat-on-support configurations interpolate exactly the effective dimensions $d_{\mathrm{eff}} = k$ of Proposition~\ref{prop:identity}, higher rungs are unstable to the pairwise rotations of Proposition~\ref{prop:ladder}(ii), and the flow of Proposition~\ref{prop:flow} implements the descent by rich-get-richer mass transport among coordinates. In Sparse TD-JEPA this pressure acts jointly with $\mathcal{L}_\phi$, $\mathcal{L}_\psi$, and the orthonormality losses of Eq.~\eqref{eq:total}, so training equilibrates at an intermediate effective dimension determined by the trade-off between predictive fit and the sparsity coefficient $\lambda_{\mathrm{sp}}$; the coefficient functions as a continuous dial on per-example effective dimensionality, rather than a hard architectural choice such as a top-$k$ constraint \citep{makhzani2013k}.
\end{remark}

\subsection{Propagation to Zero-Shot Inference}
\label{app:theory-inference}

We now connect encoder geometry to the least-squares inference of Eq.~\eqref{eq:inference}. Two preliminary observations let us analyze the unprojected least-squares solution without loss of generality.

\begin{lemma}[The final projection preserves all scale-invariant readouts]
\label{lem:proj-invariance}
Let $\hat{z} \in \mathbf{R}^d \setminus \{0\}$ and $z_r = \Pi(\hat{z})$. Then $z_r = c\,\hat{z}$ with $c = \sqrt{d}/\|\hat{z}\|_2 > 0$, so $V(s) = \psi(s)^\top z_r = c\,\psi(s)^\top \hat{z}$ for every state. Consequently every scale-invariant functional of the landscape, the normalized maps of Figures~\ref{fig:reward_landscape_ratio} and~\ref{fig:stuck_ls}, \ref{fig:lastmile_xy}, correlation coefficients and cross-validated $R^2$ against any reference signal (Figures~\ref{fig:boxplot_full}--\ref{fig:vary_n_r2}), the ranking of states by predicted reward, and the location of the maximizer, is identical for $z_r$ and $\hat{z}$.
\end{lemma}

\begin{proof}
Immediate from the definition of $\Pi$ and positive homogeneity: multiplying $V$ by a positive constant preserves ratios, orderings, argmaxima, and correlation-type statistics, and $R^2$ of a linear fit is invariant under positive rescaling of either variable.
\end{proof}

\begin{proposition}[Exact bias--variance decomposition of the inferred landscape]
\label{prop:decomp}
Fix the encoder $\psi$ and labeled states $s_1, \dots, s_N$, and let $\Psi \in \mathbf{R}^{N \times d}$ be the feature matrix of Eq.~\eqref{eq:inference} with $m := \mathrm{rank}(\Psi)$. Suppose the observed labels are $y_i = r(s_i) + \varepsilon_i$ where $r \in \mathbf{R}^N$ collects the ground-truth rewards and the $\varepsilon_i$ are independent with mean zero and variance $\sigma^2$. Let $\hat{z} = \Psi^{\dagger} y$ be the (minimum-norm) least-squares solution and $\hat{V} = \Psi \hat{z}$ the inferred rewards at the labeled states. Then
\begin{equation}
  \mathbb{E}\left[ \frac{1}{N} \bigl\| \hat{V} - r \bigr\|_2^2 \right]
  \;=\;
  \underbrace{\frac{1}{N}\bigl\| (I - P_\Psi)\, r \bigr\|_2^2}_{\textup{approximation error of the feature span}}
  \;+\;
  \underbrace{\frac{\sigma^2 m}{N}}_{\textup{estimation variance}},
  \label{eq:decomp}
\end{equation}
where $P_\Psi = \Psi \Psi^{\dagger}$ is the orthogonal projector onto the column span of $\Psi$. The first term depends only on how well the span of the learned features $\{\psi(\cdot)^\top w : w \in \mathbf{R}^d\}$, evaluated on the labeled states, can represent the reward, a property of the encoder's geometry that is independent of $N$ once the span is fixed, while the second term scales as the number of effective feature directions $m$ divided by $N$.
\end{proposition}

\begin{proof}
$P_\Psi$ is the orthogonal projector onto $\mathrm{col}(\Psi)$, with $\mathrm{rank}(P_\Psi) = m$, and $\hat V = P_\Psi y = P_\Psi r + P_\Psi \varepsilon$. Therefore
\begin{equation}
\hat V - r = -(I - P_\Psi) r + P_\Psi \varepsilon,
\end{equation}
an orthogonal decomposition ($(I-P_\Psi)r \perp P_\Psi\varepsilon$ deterministically), so
\begin{equation}
\|\hat V - r\|_2^2 = \|(I-P_\Psi) r\|_2^2 + \|P_\Psi \varepsilon\|_2^2 .
\end{equation}
Taking expectations, $\mathbb{E}\|P_\Psi\varepsilon\|_2^2 = \mathbb{E}\,\Tr\bigl(P_\Psi \varepsilon\varepsilon^\top\bigr) = \sigma^2 \Tr(P_\Psi) = \sigma^2 m$, since $P_\Psi$ is idempotent and symmetric. Dividing by $N$ gives Eq.~\eqref{eq:decomp}.
\end{proof}

\begin{corollary}[Reading the labeled-set-size sweep through Eq.~\eqref{eq:decomp}]
\label{cor:vary-n}
Consider two encoders evaluated at a common sequence of labeled-set sizes $N$, as in Analysis Experiment~2. By Eq.~\eqref{eq:decomp}, the difference between their expected landscape errors is
\begin{equation}
  \Delta(N) \;=\; \Delta_{\mathrm{approx}} \;+\; \frac{\sigma^2\,\Delta m}{N},
\end{equation}
the sum of an $N$-independent difference in approximation error and a difference in estimation variance that decays as $1/N$. The two components are therefore identifiable from the shape of the empirical gap as a function of $N$: a gap arising from estimation variance shrinks by an order of magnitude per decade of $N$, while a gap arising from the feature span's approximation error is constant in $N$. In Figure~\ref{fig:vary_n_r2}, the sparse-vs-dense gap in cross-validated $R^2$ is approximately constant across the swept decades of $N$, which corresponds to the approximation term: the penalty's effect manifests as a change in the geometry of the learned feature span itself, in the exact sense of the first term of Eq.~\eqref{eq:decomp}, rather than as a change in estimation variance at the inference stage.
\end{corollary}

\begin{remark}[Relation to the sparse-reward mechanism of Section~\ref{sec:background}]
\label{rem:mechanism}
Proposition~\ref{prop:decomp} also formalizes the informal account given in Section~\ref{sec:background} of how a diffuse encoder degrades the landscape. Coordinates of $\psi$ that vary for reasons unrelated to reward enlarge the effective span dimension $m$ entering the variance term, and, more importantly at the large $N$ used in practice, populate the feature span with directions whose projection $P_\Psi r$ misallocates reward mass across the state space; both effects lower the contrast of $V(s) = \psi(s)^\top z_r$ between goal-region states and the background. By Lemma~\ref{lem:proj-invariance} this contrast loss survives the final projection unchanged. Conversely, an encoder whose per-example activity is concentrated (Proposition~\ref{prop:identity}) can only improve the landscape through these same two terms, which is precisely the signature observed in Figures~\ref{fig:reward_landscape_ratio}, \ref{fig:boxplot_full}, and~\ref{fig:vary_n_r2}.
\end{remark}

\subsection{Gauge Symmetry}
\label{app:theory-gauge}

The results so far characterize the penalty and the inference step separately. This subsection makes precise how the two connect, by identifying an exact symmetry of the base method and locating the penalty's action relative to it. 

\begin{proposition}[Rotational gauge symmetry of the base objective]
\label{prop:gauge}
Let $R \in O(d)$. Consider the reparameterization
\begin{equation}
  \psi \mapsto R\psi, \qquad
  T_\phi \mapsto R\,T_\phi(\cdot,\cdot,R^\top\!\cdot), \qquad
  T_\psi \mapsto T_\psi(R^\top\!\cdot,\cdot,R^\top\!\cdot), \qquad
  z \mapsto Rz,
\end{equation}
applied consistently to the task encoder, both predictors' $\psi$-space inputs and outputs, the actor's task-vector conditioning, and the sampled task vectors. Then every term of the base TD-JEPA objective in Eq.~\eqref{eq:total} with $\lambda_{\mathrm{sp}} = 0$ is unchanged: the latent-prediction losses $\mathcal{L}_\phi, \mathcal{L}_\psi$ compare $\psi$-space vectors through Euclidean distances, which are preserved by $R$; the orthonormality loss depends on $\psi$ through $\mathbb{E}[\psi\psi^\top]$ compared against $I$ in a rotation-invariant matrix norm, and $\|R \Sigma R^\top - I\|_F = \|\Sigma - I\|_F$; the projection is equivariant, $\Pi(R y) = R\,\Pi(y)$; and the uniform distribution on $\mathcal{S}_d$ used for task-vector sampling is rotation-invariant. Moreover, zero-shot inference is equivariant: replacing $\Psi$ by $\Psi R^\top$ maps $\hat{z} \mapsto R\hat{z}$ and leaves every inferred value $V(s) = \psi(s)^\top \hat z$, hence every landscape, and every downstream policy retrieval, identical.
\end{proposition}

\begin{proof}[Proof sketch]
Each claim is a direct computation. For inference: with $\Psi' = \Psi R^\top$, $\hat z' = (\Psi')^\dagger y = R\,\Psi^\dagger y = R\hat z$ (using $(AR^\top)^\dagger = R A^\dagger$ for orthogonal $R$), so $\psi'(s)^\top \hat z' = \psi(s)^\top R^\top R \hat z = \psi(s)^\top \hat z$. The loss terms follow from isometry of $R$ and invariance of the Frobenius norm and of the uniform sphere measure.
\end{proof}

Proposition~\ref{prop:gauge} has two consequences that sharpen the interpretation of the penalty.

\begin{lemma}[Second-moment regularizers do not determine per-example effective dimension]
\label{lem:orth-blind}
There exist encoders with identical second moments $\mathbb{E}[\psi\psi^\top] = I$, hence identical orthonormality loss for every coefficient, whose per-example effective dimensions differ by a factor of order $d$. Concretely, let $k(s)$ be uniform on $\{1, \dots, d\}$ and $\xi(s) \in \{\pm 1\}$ unbiased, and define the axis encoder $\psi^{(1)}(s) = \sqrt{d}\,\xi(s)\, e_{k(s)}$ and, for any fixed $R \in O(d)$, the rotated encoder $\psi^{(2)} = R\,\psi^{(1)}$. Then $\mathbb{E}[\psi^{(1)}\psi^{(1)\top}] = \mathbb{E}[\psi^{(2)}\psi^{(2)\top}] = I$, while $d_{\mathrm{eff}}(\psi^{(1)}(s)) = 1$ identically and $d_{\mathrm{eff}}(\psi^{(2)}(s)) = \|R e_{k(s)}\|_1^2$, which for a uniformly random rotation concentrates at $(2/\pi)\,d\,(1 + o(1))$.
\end{lemma}

\begin{proof}
$\mathbb{E}[\psi^{(1)}\psi^{(1)\top}] = d \cdot \frac{1}{d}\sum_k e_k e_k^\top = I$, and second moments are preserved by any fixed rotation: $\mathbb{E}[\psi^{(2)}\psi^{(2)\top}] = R I R^\top = I$. The effective dimensions follow from Definition~\ref{def:effdim} with $\|\psi^{(2)}(s)\|_2 = \sqrt d$ and $\|\psi^{(2)}(s)\|_1 = \sqrt{d}\, \|R e_{k(s)}\|_1$. For Haar-random $R$, the column $Re_k$ is uniform on the unit sphere, whose coordinate marginals satisfy $\mathbb{E}|u_1| = \sqrt{2/(\pi d)}\,(1 + O(1/d))$; hence $\mathbb{E}\|Re_k\|_1 = d\,\mathbb{E}|u_1| = \sqrt{2d/\pi}\,(1 + O(1/d))$, and $\|Re_k\|_1$ concentrates around this value by standard concentration of Lipschitz functions on the sphere, giving $d_{\mathrm{eff}} = \|Re_k\|_1^2 = (2/\pi) d\, (1 + o(1))$ with high probability.
\end{proof}

\begin{remark}[Locating the penalty's action]
\label{rem:locating}
Combining the pieces: (a)~by Proposition~\ref{prop:gauge}, the base objective is exactly invariant under the rotational gauge, so its optima come in $O(d)$-orbits, and every rotation-invariant regularizer is constant on each orbit (Theorem~\ref{thm:inert}); (b)~by Lemma~\ref{lem:orth-blind}, per-example effective dimension varies from $1$ to $\Theta(d)$ \emph{within} a single orbit, so it is a quantity that neither the base losses nor any rotation-invariant addition determines; (c)~by Proposition~\ref{prop:identity}, the $\ell_1$ penalty is an exact monotone function of precisely this quantity. The penalty is thus the minimal coordinate-anisotropic addition that breaks the rotational degeneracy in the direction of concentrated embeddings. At the same time, the inference-equivariance statement in Proposition~\ref{prop:gauge} shows that re-expressing a \emph{fixed} set of learned features in a sparser basis leaves every inferred landscape unchanged; a change in downstream landscapes therefore reflects a change in the learned feature functions themselves, the span entering the approximation term of Proposition~\ref{prop:decomp}, induced by the altered training dynamics of Propositions~\ref{prop:flow} and~\ref{prop:ladder}. This is consistent with, and predicted by, the $N$-independent structure of the empirical gap in Corollary~\ref{cor:vary-n}. Finally, as stated in Section~\ref{sec:method}, per-example concentration does not by itself force different states to share a support; whether the concentrated coordinates align into a shared low-dimensional structure is an empirical property of the trained encoder, which the landscape and correlation analyses of the Experiments section probe directly. We note that the combination of pressures at work here, per-example $\ell_1$ concentration under a global second-moment isotropy constraint, is the classical recipe of sparse dictionary learning \citep{olshausen1996emergence}, which likewise uses it to align a rotationally degenerate linear representation with structure-bearing coordinates.
\end{remark}

\subsection{Alternative Regularizers Under the Sphere Constraint}
\label{app:theory-alternatives}

Table~\ref{tab:regularizer-summary} collects the conclusions of this appendix for the design space of embedding-level regularizers in this architecture. The behaviors fall into four classes: penalties that are provably inert under the sphere constraint; penalties that act only on parameter scale; penalties that favor the dense extreme; and sparsity-promoting penalties, within which $\ell_1$ is the convex member with maximal dynamic range that leaves both the architecture and the closed-form inference procedure of Eq.~\eqref{eq:inference} untouched. Hard sparsification schemes such as top-$k$ selection \citep{makhzani2013k} do enforce concentration, but do so by changing the forward computation at training and test time, replacing the continuous coefficient $\lambda_{\mathrm{sp}}$ with a discrete architectural constant and altering the feature map seen by least-squares inference; the $\ell_1$ penalty achieves graded concentration with no inference-time change, which is the property exploited by the single-line implementation of Algorithm~\ref{alg:sparse-tdjepa}. The utility of activation sparsity for TD-based value learning observed by \citet{liu2019utility} is complementary evidence from the online control setting.

\begin{table}[h]
\centering
\small
\resizebox{\textwidth}{!}{%
\begin{tabular}{llll}
\toprule
Regularizer & Behavior on $\mathcal{S}_{d}$ & Effect on training & Shown in \\
\midrule
$\ell_2$, $\|x\|_2^2$, any $h(\|x\|_2)$ & constant & none (zero gradient) & Thm.~\ref{thm:inert}, Cor.~\ref{cor:instances} \\
isotropically smoothed norms & constant & none (zero gradient) & Cor.~\ref{cor:instances} \\
weight decay (output layer) & function-preserving & parameter scale only & Rem.~\ref{rem:wd} \\
dropout on $\psi$ output & $\approx$ quadratic, $\to$ isotropic & noise; no sparsity pressure & Rem.~\ref{rem:dropout} \\
orthonormality (second moments) & orbit-invariant & blind to per-example $d_{\mathrm{eff}}(x)$ & Lem.~\ref{lem:orth-blind} \\
$\ell_p$, $2 < p \leq \infty$ & minimized at flat $x$ & favors dense embeddings & Prop.~\ref{prop:lp}(iii) \\
$\ell_p$, $1 \leq p < 2$ & minimized at sparse $x$ & favors sparse embeddings & Prop.~\ref{prop:lp}(i) \\
$\ell_p$, $p < 1$; $\ell_0$ & minimized at sparse $x$ & non-convex / no gradient & Rem.~\ref{rem:l0} \\
top-$k$ (hard) & $d_{\mathrm{eff}}(x) \leq k$ enforced & alters architecture \& inference & \S\ref{app:theory-alternatives} \\
$\ell_1$ (Eq.~\eqref{eq:sparse}) & $= \sqrt{d\, d_{\mathrm{eff}}(x)}$ exactly & continuous dial on $d_{\mathrm{eff}}(x)$ & Prop.~\ref{prop:identity} \\
\bottomrule
\end{tabular}%
}
\caption{Behavior of embedding-level regularizers under TD-JEPA's sphere constraint. $d_{\mathrm{eff}}(x)$ denotes the effective dimension of Definition~\ref{def:effdim}.}
\label{tab:regularizer-summary}
\end{table}

\section{Qualitative Rollout Analysis of the Stuck-Policy Failure Mode}
\label{app:rollouts}

RQ2 in the main paper's Experiments section asks whether the sharper, sparsity-induced reward landscape explains the stuck-policy failure mode we observe with the dense TD-JEPA baseline.
Figures~\ref{fig:stuck_ls}--\ref{fig:lastmile_xy} give concrete, single-episode illustrations of what this failure mode looks like in practice. The systematic evidence for the hypothesis itself remains the aggregate correlation results reported in the main experiment, across all seeds and domains.

Each figure shows one evaluation episode for a fully trained model, comparing Sparse TD-JEPA (\emph{Ours}, top row) against the dense TD-JEPA baseline (\emph{bottom row}) on the same task, seeded with the same evaluation seed. 
\emph{Reward Landscape} shows the spatial reward-landscape: colour encodes $z_r^\top\psi(s)$ normalized by each model's own mean predicted reward for far-field states ($>10$\,m from goal), with the goal marked by a star and dotted rings at 1\,m, 3\,m, and 5\,m.
\emph{Rollout} overlays the agent's actual trajectory on the  maze, coloured by its own learned value estimate $V(s) = \psi(s)^\top z_r$ along the path (red = low, green = high).
\emph{Value Convergence vs. Distance to Goal} plots $V(s)$ and the Euclidean distance-to-goal against timestep, together with their ordinary-least-squares trend lines; \emph{Task Success} marks with a check or cross whether the episode reached the goal.

We highlight two recurring, related instances of this failure mode below.

\subsection{Failure Mode 1: Flat Value Signal}

One failure mode we observed empirically is that the dense baseline often never acquires enough signal to leave the vicinity of its start state at all. As an example, Figures~\ref{fig:stuck_ls} and~\ref{fig:stuck_mn} show one full episode each on antmaze-large-stitch and antmaze-medium-navigate. In both cases the sparse agent's distance-to-goal decreases steadily over the episode while $V(s)$ rises correspondingly, reaching the goal well before the episode times out. The dense baseline, by contrast, remains confined to a small region near its starting position for the entire episode: its distance-to-goal trend is essentially flat, and its $V(s)$ trace is noisy with no discernible trend.

\begin{figure*}[h]
\centering
\includegraphics[width=\textwidth]{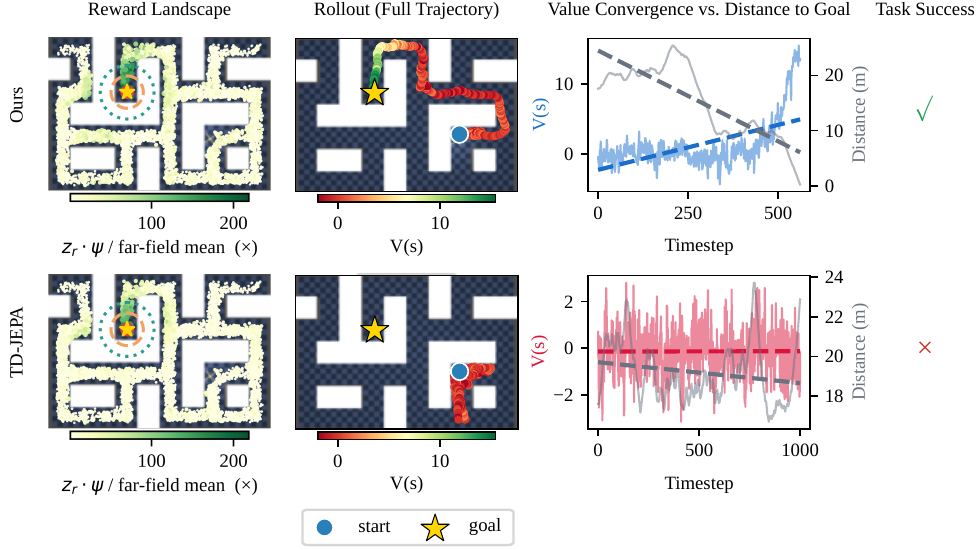}
\caption{Full antmaze-large-stitch episode (task~4). Sparse TD-JEPA navigates from its start position to the goal, with $V(s)$ rising and distance-to-goal falling together over the episode. The dense TD-JEPA baseline instead loops within a small region far from the goal for the entire episode, never making net progress, and fails.}
\label{fig:stuck_ls}
\end{figure*}

\begin{figure*}[h]
\centering
\includegraphics[width=\textwidth]{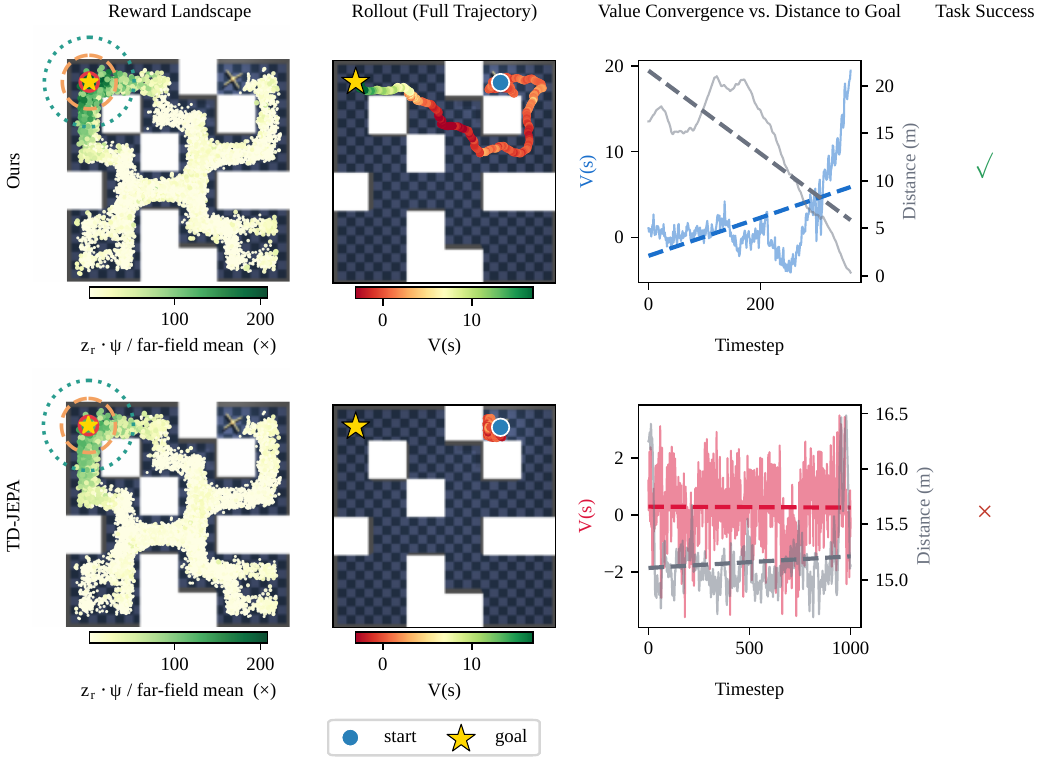}
\caption{Full antmaze-medium-navigate episode (task~4). Sparse TD-JEPA reaches the goal directly from its start position. The dense TD-JEPA baseline remains tangled in a small loop immediately around its start position for the full episode and fails to make progress toward the goal.}
\label{fig:stuck_mn}
\end{figure*}

\subsection{Failure Mode 2: Diffuse Reward Signal Near the Goal}

A related instance of the same underlying issue occurs even when the dense baseline does make it near the goal. The predicted reward stays elevated across most of the goal room rather than peaking sharply at the goal itself, giving the policy little local gradient to follow. As an example, Figure~\ref{fig:lastmile_xy} shows the last five metres of one such episode on the antmaze-large-stitch task. The sparse agent's $V(s)$ rises steadily as distance-to-goal falls, and it reaches the goal. The dense baseline's $V(s)$ instead fluctuates without a clear upward trend even as it moves in and out of the goal room, and it visibly doubles back on its own path near the goal rather than committing to it; the episode times out without success. This is consistent with the diffuse reward landscape reported in the main paper's Experiments section: near the goal, small differences in position are not reflected in $V(s)$ for the dense encoder, so the policy has no local signal telling it exactly where to stop.

\begin{figure*}[h]
\centering
\includegraphics[width=\textwidth]{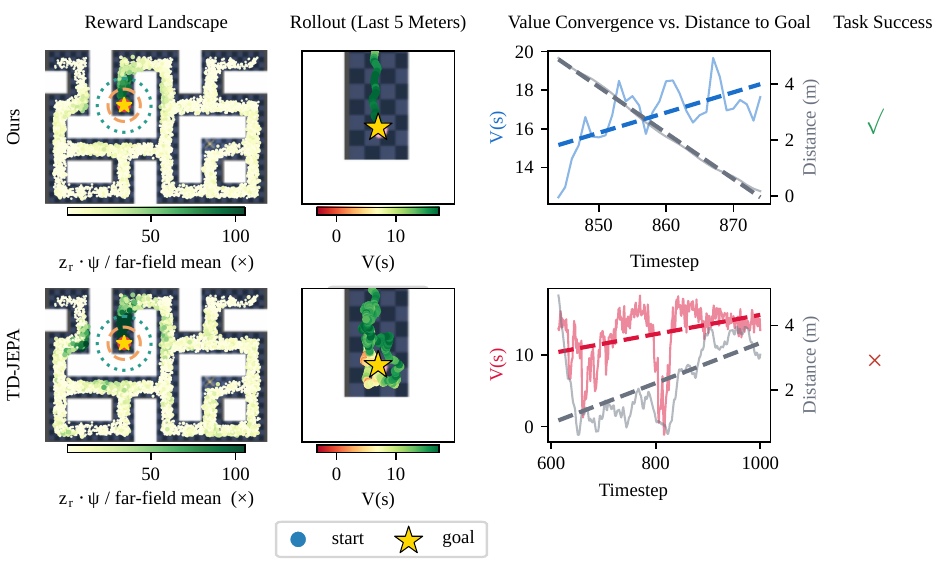}
\caption{Last five metres of an antmaze-large-stitch episode (task~4). \emph{Rollout} zoomed to the 10\,m$\times$10\,m region centred on the goal. Sparse TD-JEPA's value estimate $V(s)$ rises monotonically as it closes the remaining distance and succeeds; the dense TD-JEPA baseline's $V(s)$ is noisy and directionless over the same last-mile window, even dropping sharply to near zero at one point, and the agent wanders near the goal room without ever committing to the goal, ultimately failing the episode.}
\label{fig:lastmile_xy}
\end{figure*}

\section{Seeds and Result Aggregation}
\label{sec:seeds}

Every method reported in the main paper's results table is run with the same fixed seeds. Each seed controls network initialization and training data sampling.
\[
\begin{array}{ccccc}
3917 & 3502 & 8948 & 9460 & 4729 \\
2226 & 1744 & 7742 & 4501 & 6341
\end{array}
\] 
All 10 of these seeds are used for OGBench domains. The first 5 of these seeds are used for DMC domains. 
For each seed, success rate is measured over $10$ evaluation rollouts per task at the final training checkpoint (1M gradient steps for OGBench, 3M for DMC). 
The evaluation environment is seeded independently of the training seed, across all methods, and is set to $0$. This ensures that every model is evaluated against the same sequence of environment resets per task.
All runs use deterministic execution, which enables PyTorch's deterministic algorithms and configures CUDA Deep Neural Network library (cuDNN) to operate in deterministic mode.

\begin{figure*}[h]
\centering
\includegraphics[width=\textwidth]{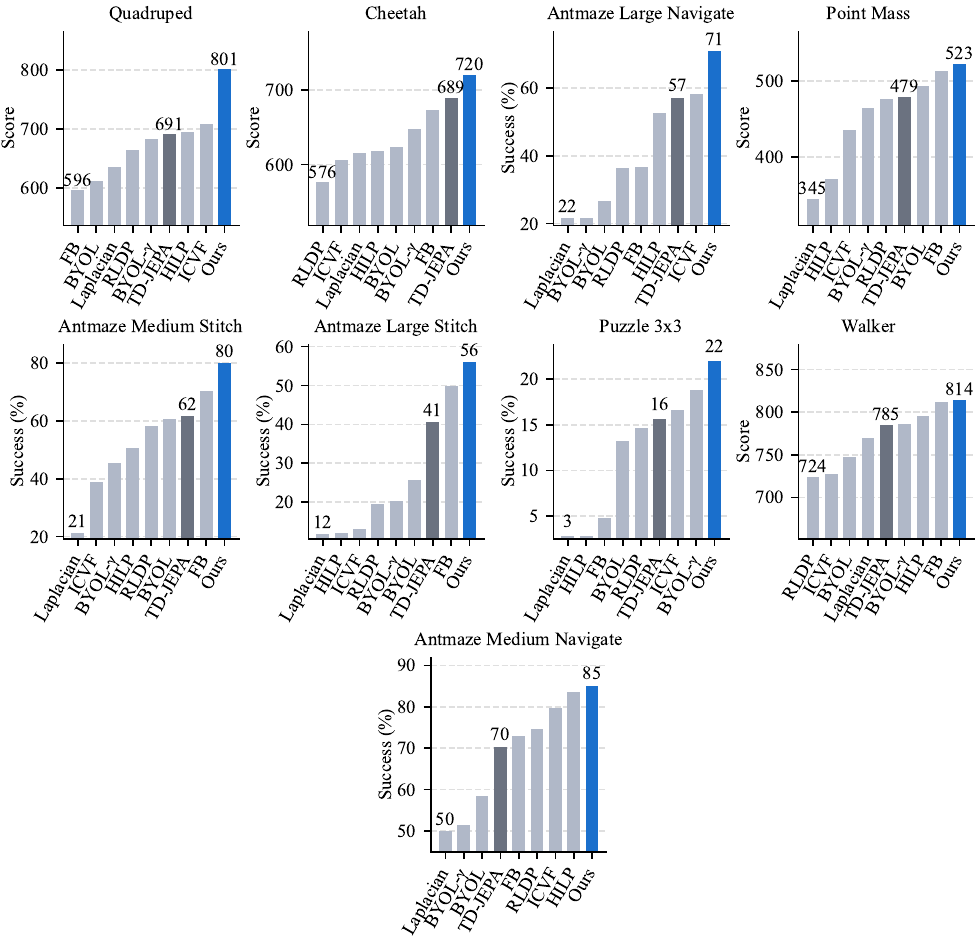}
\caption{Sparse-TD-JEPA achieves better zero-shot performance than the baseline framework in all tested domains, and outperforms all other established zero-shot baselines in most domains.}
\label{barchart}
\end{figure*}

\section{Baselines}
\label{sec:baselines-grids}

This section gives an overview of the tested baselines' architecture and learning hyperparameters. Every method below is swept over the same 10 (OGBench) / 5 (DMC) seeds listed in the \textit{Seeds and Result Aggregation} section. For each seed we keep the grid point with the highest final-checkpoint success rate. OGBench grids are $2\times2=4$ combinations per domain. DMC grids are $2\times3=6$ combinations per domain.

\subsection{Laplacian}
\cite{wu2018laplacian}
\begin{table}[H]
\centering
\small
\begin{tabular}{lll}
\toprule
Domain group & lr$_{\text{features}}$ & ortho.\ coef. \\
\midrule
OGBench Antmaze       & $\{10^{-4},10^{-5}\}$ & $\{0,\,1\}$ \\
OGBench Manipulation  & $\{10^{-4},10^{-5}\}$ & $\{0,\,1\}$ \\
DMC & $\{10^{-4},10^{-5}\}$ & $\{0,\,0.1,\,1\}$ \\
\bottomrule
\end{tabular}
\caption{Laplacian hyperparameter grid.}
\end{table}

\subsection{HILbert foundation Policy (HILP)}
\cite{park2024foundation}
\begin{table}[H]
\centering
\small
\begin{tabular}{lll}
\toprule
Domain group & lr$_{\text{features}}$ & goal-sampling prob.\ \\
\midrule
OGBench Antmaze       & $\{10^{-4},10^{-5}\}$ & $\{0.375,\,0.5\}$ \\
OGBench Manipulation  & $\{10^{-4},10^{-5}\}$ & $\{0.375,\,0.5\}$ \\
DMC & $\{10^{-4},10^{-5}\}$ & $\{0.375,\,0.2,\,0.5\}$ \\
\bottomrule
\end{tabular}
\caption{HILP hyperparameter grid (also used, unmodified, for ICVF).}
\end{table}

\subsection{Intention-Conditioned Value Function (ICVF)}
\cite{ghosh2023reinforcement}
ICVF reuses the same numeric grid as HILP (below); it differs only in fixed (non-swept) settings, namely expectile $=0.9$ (vs.\ $0.5$ for HILP) and a normalised feature encoder.

\subsection{Forward-Backward (FB)}
\cite{touati2021learning}
\begin{table}[H]
\centering
\small
\begin{tabular}{lll}
\toprule
Domain group & lr$_b$ & ortho.\ coefficient \\
\midrule
OGBench Antmaze       & $\{10^{-4},10^{-5}\}$ & $\{100,\,1000\}$ \\
OGBench Manipulation  & $\{10^{-4},10^{-5}\}$ & $\{100,\,1000\}$ \\
DMC & $\{10^{-4},10^{-5}\}$ & $\{0.1,\,1,\,10\}$ \\
\bottomrule
\end{tabular}
\caption{FB hyperparameter grid.}
\end{table}

\subsection{Regularized Latent Dynamics Prediction (RLDP)}
\cite{jajoo2026regularized}
\begin{table}[H]
\centering
\small
\begin{tabular}{lll}
\toprule
Domain group & lr$_b$ & ortho.\ coefficient \\
\midrule
OGBench Antmaze       & $\{10^{-4},10^{-5}\}$ & $\{0.001,\,0.01\}$ \\
OGBench Manipulation  & $\{10^{-4},10^{-5}\}$ & $\{0.01,\,0.1\}$ \\
DMC & $\{10^{-4},10^{-5}\}$ & $\{0.001,\,0.01,\,0.1\}$ \\
\bottomrule
\end{tabular}
\caption{RLDP hyperparameter grid.}
\end{table}

\subsection{Bootstrap Your Own Latent (BYOL)}
\cite{grill2020bootstrap}
\begin{table}[H]
\centering
\small
\begin{tabular}{lll}
\toprule
Domain group & lr$_{\text{features}}$ & ortho.\ coefficient \\
\midrule
OGBench Antmaze       & $\{10^{-4},10^{-5}\}$ & $\{0.001,\,0.01\}$ \\
OGBench Manipulation  & $\{10^{-4},10^{-5}\}$ & $\{0.01,\,0.1\}$ \\
DMC & $\{10^{-4},10^{-5}\}$ & $\{0.001,\,0.01,\,0.1\}$ \\
\bottomrule
\end{tabular}
\caption{BYOL hyperparameter grid (also used, unmodified, for BYOL-$\gamma$).}
\end{table}

\subsection{ BYOL-$\gamma$}
\cite{lawson2025self}
BYOL-$\gamma$ reuses the same numeric grid as BYOL (above); it differs only in fixed (non-swept) settings, namely multi-step targets enabled (\texttt{multi\_step=True}) and a discounted future-goal sampling horizon (\texttt{data.future}$=0.99$ OGBench / $0.98$ DMC).

\subsection{Temporal Difference Joint-Embedding Predictive Architecture (TD-JEPA)}
\cite{bagatella2025td}
\begin{table}[H]
\centering
\small
\begin{tabular}{lll}
\toprule
Domain group & lr$_\psi$ & $\psi_{\text{ortho}}$-coefficient \\
\midrule
OGBench Antmaze       & $\{10^{-4},10^{-5}\}$ & $\{0.1,\,1\}$ \\
OGBench Manipulation  & $\{10^{-4},10^{-5}\}$ & $\{1,\,10\}$ \\
DMC & $\{10^{-4},10^{-5}\}$ & $\{0.01,\,0.1,\,1\}$ \\
\bottomrule
\end{tabular}
\caption{TD-JEPA hyperparameter grid.}
\end{table}

\section{Dataset and Evaluation Environment Details}
\label{app:datasets}

We evaluate our method on 65 tasks across the standard zero-shot reinforcement learning benchmarking frameworks, the \textit{DeepMind Control Suite} (DMC) \cite{tassa2018deepmind,tassa2020dmcontrol} and the \textit{Offline Goal-Conditioned RL Benchmark} (OGBench) \cite{park2025ogbench}. Task descriptions and dimensions for DMC are listed in detail in Table~\ref{tab:dmc-specs}.  Task descriptions and dimensions for OGBench are listed in detail in Table~\ref{tab:ogbench-specs}.

\begin{table*}[h]
\centering
\small
\caption{DMC dataset domains, dimensions, and tasks.}
\label{tab:dmc-specs}
\resizebox{\textwidth}{!}{%
\begin{tabular}{lccll}
\toprule
Domain & State dim. & Action dim. & Tasks & Task Description \\
\midrule
\multirow{4}{*}{Walker}    & \multirow{4}{*}{24} & \multirow{4}{*}{6}  & Stand & Upright torso, minimal height \\
 & & & Walk & Stand reward plus forward velocity \\
 & & & Run & Stand reward plus higher velocity \\
 & & & Spin & Continuous rotation while balanced \\
\midrule
\multirow{4}{*}{Cheetah}   & \multirow{4}{*}{17} & \multirow{4}{*}{6}  & Walk & Reward linear in forward velocity, max 10 m/s \\
 & & & Walk Backward & Reward linear in backward velocity \\
 & & & Run & Reward linear in forward velocity, max 10 m/s \\
 & & & Run Backward & Backward velocity, higher target speed \\
\midrule
\multirow{8}{*}{Pointmass} & \multirow{8}{*}{4}  & \multirow{8}{*}{2}  & Reach Top Left & Navigate to top-left target \\
 & & & Reach Top Right & Navigate to top-right target \\
 & & & Reach Bottom Left & Navigate to bottom-left target \\
 & & & Reach Bottom Right & Navigate to bottom-right target \\
 & & & Reach Bottom Left Long & Navigate to distant bottom-left target \\
 & & & Loop & Trace a looped trajectory \\
 & & & Square & Trace a square trajectory \\
 & & & Fast Slow & Alternate fast and slow movement segments \\
\midrule
\multirow{4}{*}{Quadruped} & \multirow{4}{*}{78} & \multirow{4}{*}{12} & Stand & Maintain balanced standing posture \\
 & & & Walk & Locomote at moderate speed \\
 & & & Run & Locomote at higher speed \\
 & & & Jump & Jump in place \\
\bottomrule
\end{tabular}%
}
\end{table*}

\begin{table*}[h]
\centering
\small
\caption{OGBench dataset domains, dimensions, and tasks.}
\label{tab:ogbench-specs}
\resizebox{\textwidth}{!}{%
\begin{tabular}{lccll}
\toprule
Domain & State dim & Action dim & Task & Task Description \\
\midrule
Antmaze Medium Navigate & 29 & 8 & Task1--5 & Navigate to five goals in the medium maze \\
\midrule
Antmaze Large Navigate   & 29 & 8 & Task1--5 & Navigate to five goals in the large maze \\
\midrule
Antmaze Medium Stitch   & 29 & 8 & Task1--5 & Reach five goals by stitching sub-trajectories \\
\midrule
Antmaze Large Stitch    & 29 & 8 & Task1--5 & Reach five goals by stitching sub-trajectories \\
\midrule
Antmaze Medium Explore  & 29 & 8 & Task1--5 & \shortstack[l]{Reach five goals from low-quality,\\high-coverage exploratory data} \\

\midrule
\multirow{5}{*}{Cube Single}            & \multirow{5}{*}{28} & \multirow{5}{*}{5} & Horizontal & Pick-and-place, horizontal offset \\
 & & & Vertical1 & Pick-and-place, vertical offset \\
 & & & Vertical2 & Pick-and-place, vertical offset \\
 & & & Diagonal1 & Pick-and-place, diagonal offset \\
 & & & Diagonal2 & Pick-and-place, diagonal offset \\
\midrule
\multirow{5}{*}{Cube Double}             & \multirow{5}{*}{37} & \multirow{5}{*}{5} & Single Pick-and-Place & Move one cube \\
 & & & Double Pick-and-Place 1 & Move both cubes \\
 & & & Double Pick-and-Place 2 & Move both cubes \\
 & & & Swap & Exchange cube positions \\
 & & & Stack & Stack one cube on the other \\
\midrule
\multirow{5}{*}{Scene}                   & \multirow{5}{*}{40} & \multirow{5}{*}{5} & Open & Open drawer or window \\
 & & & Unlock-and-Lock & Toggle button locks \\
 & & & Rearrange Medium & Reconfigure objects \\
 & & & Put in Drawer & Unlock, open, place, close \\
 & & & Rearrange Hard & Reconfigure, up to 8 steps \\
\midrule
Puzzle 3x3              & 55 & 5 & Task1--5 & Press buttons to reach five color targets \\
\bottomrule
\end{tabular}%
}
\end{table*}

\end{document}